\documentclass{article}
\pdfoutput=1

\usepackage[final]{neurips_2023}

\usepackage[utf8]{inputenc} % allow utf-8 input
\usepackage[T1]{fontenc}    % use 8-bit T1 fonts
\usepackage{hyperref}       % hyperlinks
\usepackage{url}            % simple URL typesetting
\usepackage{booktabs}       % professional-quality tables
\usepackage{amsfonts}       % blackboard math symbols
\usepackage{nicefrac}       % compact symbols for 1/2, etc.
\usepackage{microtype}      % microtypography
\usepackage{xcolor}         % colors
\usepackage{graphicx}
\usepackage{subfigure}
\usepackage{amsmath}
\usepackage{amssymb}
\usepackage{mathtools}
\usepackage{amsthm}
\usepackage{algorithm}
\usepackage{algorithmic}
\usepackage{multicol}
\usepackage{tabularx}
\usepackage{wrapfig}
\usepackage{nameref}

\title{Accelerating Monte Carlo Tree Search with Probability Tree State Abstraction}

\author{
Yangqing~Fu
\qquad
Ming~Sun 
\qquad
Buqing~Nie
\qquad
Yue~Gao\thanks{Corresponding author. 
}\\
MoE Key Lab of Artificial Intelligence, AI Institute, Shanghai Jiao Tong University\\
{\tt \{frank79110, mingsun, niebuqing, yuegao\}@sjtu.edu.cn}}

\begin{document}

\newtheorem{theorem}{Theorem}[section]
\newtheorem{proposition}[theorem]{Proposition}
\newtheorem{lemma}[theorem]{Lemma}
\newtheorem{corollary}[theorem]{Corollary}
\theoremstyle{definition}
\newtheorem{definition}[theorem]{Definition}
\newtheorem{assumption}[theorem]{Assumption}

\newtheorem{remark}[theorem]{Remark}

\maketitle

% \begin{abstract}
% Monte Carlo Tree Search (MCTS) algorithms such as AlphaGo and MuZero have achieved superhuman performance in many challenging tasks. However, MCTS-based algorithms require multiple simulations to find high-quality solutions in complex tree search spaces. To improve node expansion efficiency and reduce search complexity, a general state abstraction and information-balanced sampling (\textit{AbS}) framework for MCTS-based algorithms is proposed. The proposed framework utilizes a novel \textit{information-balanced sampling} method to improve sample efficiency during MCTS expansion. The expansion balances immediate reward with approximate information gain. Furthermore, a general tree state abstraction is formulated to decrease the branching factor for tree search problems, with a theoretical guarantee of the transitivity and bounded aggregation error. Experimental results demonstrate that our method can improve the computation efficiency of state-of-the-art MCTS-based algorithms with significant search space reduction. 
% \end{abstract}

\begin{abstract}
Monte Carlo Tree Search (MCTS) algorithms such as AlphaGo and MuZero have achieved superhuman performance in many challenging tasks. However, the computational complexity of MCTS-based algorithms is influenced by the size of the search space. To address this issue, we propose a novel probability tree state abstraction (\textit{PTSA}) algorithm to improve the search efficiency of MCTS. A general tree state abstraction with path transitivity is defined. In addition, the probability tree state abstraction is proposed for fewer mistakes during the aggregation step. Furthermore, the theoretical guarantees of the transitivity and aggregation error bound are justified. To evaluate the effectiveness of the \textit{PTSA} algorithm, we integrate it with state-of-the-art MCTS-based algorithms, such as Sampled MuZero and Gumbel MuZero. Experimental results on different tasks demonstrate that our method can accelerate the training process of state-of-the-art algorithms with $10\%-45\%$ search space reduction.

\end{abstract}

\section{Introduction}
\label{introduction}

With the advancement of reinforcement learning (RL), AlphaGo is the first algorithm that can defeat human professional players in the game of Go \cite{silver2016mastering}. With no supervised learning of expert moves in AlphaGo Zero \cite{silver2017alphagozero} to restructured self-play in AlphaZero \cite{silver2017mastering}, until the recent MuZero that conducts Monte Carlo Tree Search (MCTS) in hidden state space with a learned dynamics model \cite{schrittwieser2020mastering}. MuZero presents a powerful generalization framework that allows the algorithm to learn without a perfect simulator in complex tasks. Recently, EfficientZero \cite{ye2021mastering} has made significant progress in the sample efficiency of MCTS-based algorithms. This development has opened up new possibilities for real-world applications, including robotics and self-driving.

% Real-world problems with large discrete or continuous action spaces are often complex, and MCTS algorithms require a large branching factor to explore the tree space. 
% When dealing with some real-world problems, MCTS-based algorithms require a large branching factor to explore the tree space. The sample-based method is one of the approaches to address this problem \cite{stadie2018importance, osband2013more}. For instance, the sample-based policy iteration method \textit{Sampled MuZero} performs policy improvement and evaluation by sampling from a small subset of the action space \cite{hubert2021learning}. Sample-based methods can also balance exploration with exploitation. Instead of obtaining an immediate high reward, \textit{information-directed sampling} method aims at finding a policy that minimizes the information ratio to balance exploration with exploitation in the next episode \cite{russo2014learning,lu2019information}. 

% To address complicated state representations, a large number of simulations are required for MCTS \cite{sun2021research,zhao2022mcts}. 

When dealing with complex decision-making problems, increasing the search depth is necessary to achieve more accurate exploration in the decision space, but this also leads to higher time and space complexity \cite{sun2021research,zhao2022mcts}. For instance, MuZero trained for 12 hours with 1000 TPUs to learn the Go game, and for a single Atari game, it needs 40 TPUs for 12 hours of training \cite{schrittwieser2020mastering}. One approach to reduce the computation is the state abstraction method, which aggregates states based on a certain similarity measure to obtain a near-optimal policy \cite{majeed2019performance,hutter2016extreme, abel2017toward}. state abstraction is a crucial technique in reinforcement learning (RL) that enables efficient planning, exploration, and generalization \cite{abel2018state}. 
% When directly applying previous state abstraction methods to MCTS-based algorithms, we will face the challenge caused by the branching and parallel computation of the tree search.

To reduce the search space of MCTS, previous studies have shown the potential of specific state abstraction techniques \cite{hostetler2014state, bai2016markovian, sokota2021monte, dockhorn2021game}. 
However, finding the minimum abstract state space in these studies is an NP-Hard problem \cite{even2003approximate}. To our knowledge, this work is the first to define path transitivity in the formulation of tree state abstraction, which enables the discovery of the minimum abstract state space in polynomial time. Additionally, recent MCTS-based algorithms utilize neural networks to estimate the value or reward of states, which may lead to errors in aggregating states with deterministic state abstraction functions. To address this issue, we proposed a probability tree state abstraction function that aggregates states based on the distribution of child node values, which enhances the robustness of aggregation and ensures transitivity.

This paper proposes the probability tree state abstraction (\textit{PTSA}) algorithm to improve the tree search efficiency of MCTS. The main contributions can be summarized as follows: i) A general tree state abstraction is formulated, and path transitivity is also defined in the formulation. ii) The probability tree state abstraction is proposed for fewer mistakes during the aggregation step. iii) The theoretical guarantees of the transitivity and aggregation error bound are justified. iv) We integrate \textit{PTSA} with state-of-the-art algorithms and achieve comparable performance with $10\%-45\%$ reduction in tree search space.

\section{Related Work}
\subsection{MCTS-based Methods}
MCTS is a rollout algorithm for solving sequential decision problems \cite{sutton2018reinforcement}. The fundamental idea of MCTS is to search for the most promising actions by randomly sampling the search space, and then expanding the search tree based on those actions \cite{browne2012survey}. The computational bottlenecks arise from the search loop, especially interacting with the real environment model of each iteration.

Combined with deep neural networks, MCTS-based methods have achieved better performance and efficiency in various complex tasks, such as board games \cite{silver2017alphagozero}, autonomous driving \cite{chen2020driving}, and robot planning \cite{schrittwieser2021online}. The model-based algorithm \textit{MuZero} \cite{schrittwieser2020mastering} predicts the environmental dynamics model for more efficient simulation. Based on \textit{MuZero}, \textit{EfficientZero} \cite{ye2021mastering} is proposed for the training with limited data, which achieves super-human performance on Atari 100K benchmarks. However, both \textit{MuZero} and \textit{EfficientZero} require high computational consumption when dealing with complex action spaces. To address arbitrarily complex action spaces, the sample-based policy iteration framework \cite{hubert2021learning} is proposed. \textit{Sampled MuZero} extends \textit{MuZero} by sub-sampling a small fraction of possible moves and achieves higher sample efficiency with fewer expanded actions and simulations. The experimental results show that planning over the sampled tree provides a near-optimal approximation \cite{hubert2021learning}. To further reduce the number of simulations, \textit{Gumbel MuZero} utilizes the Gumbel-Top-k trick to construct efficient planning \cite{danihelka2022policy}. 

\subsection{State Abstraction}
State abstraction is aimed at reducing the complex state space by aggregating the similar states \cite{abel2018state}. The original state space $S$ can be mapped into a smaller abstract state space $S_{\phi}$ by state abstraction. By grouping similar states together, state abstraction can help to identify patterns and regularities in the environment, which can inform more effective decision-making \cite{hostetler2014state}.

There are two main challenges when applying state abstraction to RL problems. The first challenge is to decrease the value loss between $S$ and $S_{\phi}$. With bounded value loss, the approximate state abstractions allow the agent to learn a near-optimal policy with improved training efficiency \cite{abel2018state, abel2016near}. The second challenge is to compute the smallest possible abstract state space, which is proven that the computational complexity is NP-hard \cite{even2003approximate}. The transitive state abstraction \cite{abel2018state} is defined to efficiently compute the smallest possible abstract state space. However, most transitive state abstractions with deterministic predicates have low fault tolerance. To improve the robustness of aggregation, we measure the abstraction probability of the state pairs based on the expected value distributions. In addition, some previous studies have analyzed and discussed some specific state abstractions in tree structure \cite{hostetler2014state, bai2016markovian, sokota2021monte, dockhorn2021game}, but there is no formal definition and analysis for general tree state abstractions. To our knowledge, \textit{PTSA} is the first method that defines a general formulation of tree state abstractions for deep MCTS-based methods and analyzes the transitivity and aggregation error.

\section{Preliminaries} 

In this section, the prerequisites for our proposed method are introduced. We consider an agent learning in a Markov decision process (MDP) represented as $\langle \mathcal{S}, \mathcal{A}, R, \mathcal{T}, \gamma\rangle$, where $\mathcal{S}$ denotes the state space, $\mathcal{A}$ denotes the action space, $R: \mathcal{S} \times \mathcal{A} \mapsto \mathbb{R}$ denotes the reward function, $\mathcal{T}: \mathcal{S} \times \mathcal{A} \mapsto \mathbb{P}(\mathcal{S})$ denotes the transition model, and $\gamma \in [0,1]$ is the discount factor. The goal is to learn a policy $\pi: \mathcal{S} \mapsto \mathbb{P}(\mathcal{A})$ that maximizes the long-term expected reward in the MDP.

\subsection{Monte Carlo Tree Search}

MCTS-based algorithms typically involve four stages in the search loop: selection, expansion, simulation, and backpropagation. After $N$ iterations of the search loop, MCTS generates a policy based on the current states. In the selection stage of each iteration, an action is selected by maximizing over UCB. Notably, AlphaZero \cite{silver2017alphagozero} and MuZero \cite{schrittwieser2020mastering}, two successful RL algorithms developed based on a variant of Upper Confidence Bound (UCB) \cite{gelly2006exploration} called probabilistic upper confidence tree (PUCT) \cite{rosin2011multi}, have achieved remarkable results in board games and Atari games. The formula of PUCT is given by Eq. \eqref{PUCT}:
\begin{equation}
\label{PUCT}
a^{k} = \operatorname{argmax}_{a} Q(s, a)+c(s) \cdot P(s, a) \frac{\sqrt{\sum_{b} N(s, b)}}{1+N(s, a)},
\end{equation}
where $k$ is the index of iteration, $Q(s, a)$ denotes the value of action $a$ in state $s$, $c(s)$ is a hyperparameter for balancing the value score with the visiting counts $N(s, a)$, and $P(s, a)$ is the policy prior obtained from neural networks.

\subsection{State Abstraction in RL}
State abstraction methods aggregate similar environment states to compressed descriptions \cite{andre2002state}, which simplify the state spaces and significantly reduce the computation time \cite{anand2016oga}. The state abstraction type is formulated as below \cite{abel2017toward}:

\begin{definition}
    (State Abstraction Type) A state abstraction type is a set of functions $\phi:\mathcal{S} \mapsto \mathcal{S}_{\phi}$ related to a fixed predicate on state pairs: $p_{M}: \mathcal{S} \times \mathcal{S} \mapsto\{0,1\}$.
    % \begin{equation}
    %     p_{M}: \mathcal{S} \times \mathcal{S} \mapsto\{0,1\}
    % \end{equation}
    When function $\phi$ aggregates the state pair $(s_{1}, s_{2})$ in MDP $M$, the predicate $p_{M}$ must be true: $\phi\left(s_{1}\right)=\phi\left(s_{2}\right) \Longrightarrow p_{M}\left(s_{1}, s_{2}\right)$.
    % \begin{equation}
    %     \phi\left(s_{1}\right)=\phi\left(s_{2}\right) \Longrightarrow p_{M}\left(s_{1}, s_{2}\right)
    % \end{equation}
\end{definition}

% A general RL problem with state abstraction function $\phi$ can be described in Figure \ref{fig:GABRL}. Interpreting from the information theory perspective, the state abstraction function $\phi$ compresses the original state space $\mathcal{S}$ to the abstract state space $\mathcal{S_{\phi}}$, which aggregates the state pairs that have equivalent and useful information.
 
The conditions for state abstraction are usually strict, which may cause insufficient compression in state spaces. Recent studies have shown that transitive state abstraction \cite{abel2016near} is computationally efficient and can achieve near-optimal decision-making, which is defined as:

\begin{definition}
    (Transitive State Abstraction) For a given state abstraction $\phi$ with predicate $p_{M}$, if $ {\forall} (s_{1}, s_{2}, s_{3}) \in \mathcal{S}$ satisfies $\left[p_{M}\left(s_{1}, s_{2}\right) \wedge p_{M}\left(s_{2}, s_{3}\right)\right] \Longrightarrow p_{M}\left(s_{1}, s_{3}\right)$, the state abstraction $\phi$ is a transitive state abstraction.
    % \begin{equation}
    %     \left[p_{M}\left(s_{1}, s_{2}\right) \wedge p_{M}\left(s_{2}, s_{3}\right)\right] \Longrightarrow p_{M}\left(s_{1}, s_{3}\right)
    % \end{equation}
\end{definition}

% \begin{wraptable}{R}{0.7\textwidth}
% \vskip -0.6cm
% \caption{Some previous state abstraction functions \cite{abel2016near,abel2018state,li2006towards}.}

% \label{tab: previous state abstractions}
% \begin{center}
% % \begin{small}
% % \begin{sc}
% \centering
% \scalebox{0.95}{
% \centering
% \begin{tabular}{lcc}
% \toprule
% Abstractions &Predicate  &Transitive\\
% \midrule 						
% $\phi_{a^{*}} $ &$a^{*}_{1} = a^{*}_{2} \wedge V^{*}({s_{1}})=V^{*}({s_{2}})$  & yes \\
% $\phi_{a^{*}}^{\varepsilon}$ & $a^{*}_{1} = a^{*}_{2} \wedge | V^{*}({s_{1}})-V^{*}({s_{2}})| \leq \varepsilon$ &no \\
% $\phi_{Q^{*}}$	& $\max _{a}\left|Q_{M}^{*}\left(s_{1}, a\right)-Q_{M}^{*}\left(s_{2}, a\right)\right| = 0$ &yes\\
% $\phi_{Q^{*}}^{\varepsilon} $	&  $\max _{a}\left|Q_{M}^{*}\left(s_{1}, a\right)-Q_{M}^{*}\left(s_{2}, a\right)\right| \leq \varepsilon $ & no \\
%  $\phi_{Q^{*}_{d}}$	& $\forall_{a}:\left\lceil\frac{Q^{*}\left(s_{1}, a\right)}{d}\right\rceil=\left\lceil\frac{Q^{*}\left(s_{2}, a\right)}{d}\right\rceil$ &yes \\
% \bottomrule
% %\vskip -0.5cm
% \end{tabular}
% }
% % \end{small}
% \end{center}
% \vskip -0.5cm
% \end{wraptable}
Some previous state abstraction functions are shown in Table \ref{tab: previous state abstractions}, which aim to abstract similar states in general reinforcement learning methods. For instance, abstraction $\phi_{a^{*}}$ considers the optimal actions and values of states, which is also widely studied in MCTS-based methods \cite{lan2022alphazero}.

Transitive state abstraction can reduce the computational cost of finding the smallest abstract state space. As shown in Table \ref{tab: previous state abstractions}, most approximate state abstractions are not transitive. However, transitivity is challenging to MCTS-based methods due to the tree structure. In MCTS, a path may contain multiple states, so it is necessary to derive the transitivity of the path from the transitivity of the states. Finding the smallest abstract space in the search tree becomes an NP-hard problem if the state abstraction function lacks transitivity in the path. 

 \begin{table}
% \vskip -0.6cm
\caption{Some previous state abstraction functions \cite{abel2016near,abel2018state,li2006towards}.}

\label{tab: previous state abstractions}
\begin{center}
% \begin{small}
% \begin{sc}
\centering
\scalebox{0.95}{
\centering
\begin{tabular}{lcc}
\toprule
Abstractions &Predicate  &Transitive\\
\midrule 						
$\phi_{a^{*}} $ &$a^{*}_{1} = a^{*}_{2} \wedge V^{*}({s_{1}})=V^{*}({s_{2}})$  & yes \\
$\phi_{a^{*}}^{\varepsilon}$ & $a^{*}_{1} = a^{*}_{2} \wedge | V^{*}({s_{1}})-V^{*}({s_{2}})| \leq \varepsilon$ &no \\
$\phi_{Q^{*}}$	& $\max _{a}\left|Q_{M}^{*}\left(s_{1}, a\right)-Q_{M}^{*}\left(s_{2}, a\right)\right| = 0$ &yes\\
$\phi_{Q^{*}}^{\varepsilon} $	&  $\max _{a}\left|Q_{M}^{*}\left(s_{1}, a\right)-Q_{M}^{*}\left(s_{2}, a\right)\right| \leq \varepsilon $ & no \\
 $\phi_{Q^{*}_{d}}$	& $\forall_{a}:\left\lceil\frac{Q^{*}\left(s_{1}, a\right)}{d}\right\rceil=\left\lceil\frac{Q^{*}\left(s_{2}, a\right)}{d}\right\rceil$ &yes \\
\bottomrule
%\vskip -0.5cm
\end{tabular}
}
% \end{small}
\end{center}
% \vskip -0.5cm
\end{table}

\begin{wrapfigure}{R}{0.55\textwidth}
    \centering
    % \vskip -0.5cm
    \includegraphics[width=0.55\textwidth]{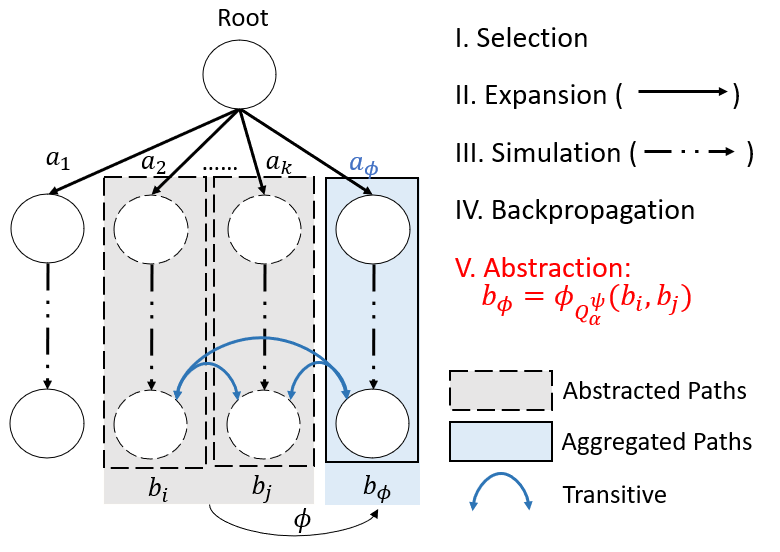}
    % \vskip -0.25cm
    \caption{The overview structure of \textit{PTSA} algorithm. The original tree search space in MCTS is mapped into a smaller abstract space efficiently by transitive probability tree state abstraction $\phi_{Q^{\psi}_{\alpha}}$.}
    \label{fig:AbS}
    \vskip -0.5cm
\end{wrapfigure}

\section{PTSA-MCTS}
In this section, we introduce the proposed probability tree state abstraction (\textit{PTSA}) algorithm, which improves the tree search efficiency for MCTS. As described in Figure \ref{fig:AbS}, \textit{PTSA} algorithm improves the search efficiency of MCTS in two aspects: one is reducing the search space by abstracting the original search space, and the other is finding the smallest abstract space efficiently by transitive probability tree state abstraction $\phi_{Q^{\psi}_{\alpha}}$. The organization of this section is as follows: Subsection \ref{sec:tsa} gives the formulation of general tree state abstraction. Subsection \ref{sec:PTSA} presents our proposed probability tree state abstraction. Subsection \ref{sec:algorithm}  presents the PTSAZero algorithm, which integrates \textit{PTSA} with Sampled MuZero, and more information can be found in the Appendix. Subsection \ref{sec:Theo} provides the proofs of transitivity in tree structures and bounded aggregation error under balanced exploration.

\subsection{Tree State Abstraction Formulation}\label{sec:tsa}

The formulation of our tree state abstraction is provided with a general abstraction operator for the tree structure, which can be utilized with arbitrary state abstraction types. To facilitate the derivation of theorems and properties, the definition of tree state abstraction is given:

\begin{definition}
    (Tree State Abstraction) 

    For a given tree, $\mathcal{V}$ and $\mathcal{B}$ denote the node set and path set respectively.
    A tree state abstraction is a function $\phi : \mathcal{V} \mapsto \mathcal{V}_{\phi} \And \mathcal{B} \mapsto \mathcal{B}_{\phi}$ with node predicate $p_{vM}$ and path predicate $p_{bM}$  on the sibling path pair: 
    \begin{equation}
    \begin{split}
        & p_{vM}: \mathcal{V} \times \mathcal{V} \mapsto\{0,1\}; \\
        & p_{bM}: \mathcal{B} \times \mathcal{B} \mapsto\{0,1\}.
    \end{split}
    \end{equation}
    In the tree structure, a path is a sequence of nodes and a node denotes the representation of the corresponding state. The path and node predicates abstract corresponding states of the path and node respectively. For a given path pair $(b_{1},b_{2})$ with the same length $l$, the predicate on this path pair can be decomposed as: 
    \begin{equation}
        p_{bM}(b_{1},b_{2}) = p_{vM}(v_{0}^{b_{1}}, v_{0}^{b_{2}}) \wedge \cdots \wedge p_{vM}(v_{l-1}^{b_{1}}, v_{l-1}^{b_{2}}),
    \end{equation}
    where $v_{i-1}^{b} (i = 1,...,l)$ is the $i$-th node of path $b$ with length $l$. When function $\phi$ aggregates the path pair $(b_{1}, b_{2})$ in MDP $M$, the predicate $p_{M}$ must be true:
     \begin{equation}
        \phi\left(b_{1}\right)=\phi\left(b_{2}\right) \Longrightarrow p_{bM}\left(b_{1}, b_{2}\right),
    \end{equation}
    and path $b_1$ and path $b_2$ belong to the same abstract cluster. 
\end{definition}

The tree state abstraction is applied to two search paths of equal length that start from the same parent node. This ensures that each node along the paths has a corresponding potential abstracted node, while preserving the Markov property after aggregation. For instance, the path pair $(b_1,b_2)$ from different parent nodes $v_a, v_b$ can not be aggregated, which violates the Markov property in the state transitions $v_a \rightarrow v^{b_1}_{0} \rightarrow \ldots\rightarrow v^{b_1}_{l-1}$ and $v_b \rightarrow v^{b_2}_{0} \rightarrow \ldots\rightarrow v^{b_2}_{l-1}$. 

Previous studies have proven that tree state abstraction can be an efficient approach to reducing branching factors in MCTS \cite{hostetler2014state,anand2016oga,sokota2021monte}. However, finding the smallest abstract state space is an NP-Hard problem in the previous studies \cite{even2003approximate}. Following the definition of transitive state abstraction \cite{abel2018state}, we define the path transitivity in the tree state abstraction formulation:

\begin{definition}
    (Path Transitivity) For a given tree state abstraction $\phi$ with predicate $p_{bM}$, if $ {\forall} (b_{1}, b_{2}, b_{3}) \in \mathcal{B}$ satisfies $\left[p_{bM}\left(b_{1}, b_{2}\right) \wedge p_{bM}\left(b_{2}, b_{3}\right)\right] \Longrightarrow p_{bM}\left(b_{1}, b_{3}\right)$, the state abstraction $\phi$ has path transitivity.
\end{definition}

The definition of path transitivity extends the equivalence of state abstraction in tree state space and general RL state space. Tree State abstraction for MCTS maps the original path space $\mathcal{B}$ into the abstract path space $\mathcal{B_{\phi}}$. For non-transitive tree state abstraction, it is necessary to determine whether all possible path pairs belong to the same abstract cluster, and paths may appear repeatedly in path pairs, which requires a massive computation cost to obtain the smallest $\mathcal{B_{\phi}}$. 

\begin{algorithm}[t!]
\caption{PTSAZero}
\label{alg:TSAES}
\setlength{\multicolsep}{2pt}
\begin{multicols}{2} \raggedcolumns
\begin{algorithmic}[1]
\STATE \textbf{Input:} Root node $v_0$, dynamics network $\mathbf{d}_{\theta}$, policy network $\mathbf{p}_{\theta}$, value network $\mathbf{v}_{\theta}$
\STATE Initialize searched-path list $S_{L}=\{[v_0] \}$
\FOR {$n = 0,1,2...$}
\STATE Reset $v = v_0$, searching branch $b_{s} = [v]$
    % \WHILE {$v.child \neq \varnothing$} 
        % \STATE $\hat{a}={\operatorname{argmax}}_{a \in \mathcal{A}} \  \textbf{PUCT}(v,a)$ 
        % \STATE Child node $v'$ =  $v.select\_child(\hat{a})$ 
        % \STATE Add child into $b_{s}$: $b_{s}.add(v), v = v'$
        \STATE Selection with \textbf{PUCT}
        % \STATE Add child $v'$ into $b_{s}$: $b_{s}.add(v'), v = v'$
        \STATE Add child $v'$ into $b_{s}$, $v = v'$
    % \ENDWHILE
\STATE Update hidden state $h$ and reward $r$:
\STATE \quad \quad \ \  $h$, $r= \mathbf{d}_{\theta}(v.parent.h,a)$
\STATE Predict value $V$ and policy $\pi$:
\STATE \quad \quad \ \ $V, \pi = \mathbf{v}_{\theta}(v.h),\mathbf{p}_{\theta}(v.h)$
% \STATE Calculate proposal distribution according to Eq.  \eqref{IBS}:\\
% \STATE Compute $\pi'$ according to Eq. \eqref{IBS} 
% \STATE $ \pi' := \underset{\pi_{m}}{\arg \min } \frac{\left(\pi_{m}^{\mathsf{T}} \tilde{\Delta}_{t}(h,a)\right)^{2}}{\pi_{m}^{\mathsf{T}}\tilde{g}_{t}(h,a)}$
\STATE Expand child nodes: $v.expand(\pi, h, r)$ 
% \FOR {$v_{i} \in reverse(b_{s})$} 
\STATE Backpropagation along path $b_{s}$
    % \STATE $v_{i}.value \  +\!= V$
    % \STATE $V \  +\!= v_{i}.r + \gamma V$
% \ENDFOR
\STATE $S_{L}.add(b_{s})$
\FOR {$b_{i} \in S_{L}$} 
    \IF{$\phi_{Q^{\psi}_{\alpha}} (b_{i}) = \phi_{Q^{\psi}_{\alpha}} (b_{s})$}
        \STATE $b_{j} = \underset{b \in (b_{i},b_{s})}{\arg \min }(b.V)$ 
        \STATE $v_0.pruning(b_{j})$, $S_{L}.delete(b_{j})$
    % \ELSE
    \ENDIF
\ENDFOR
\ENDFOR
% \STATE \textbf{return} $S_{L}$

\end{algorithmic}
\end{multicols}
\end{algorithm}

\subsection{Probability Tree State Abstraction}\label{sec:PTSA}
State-of-the-art MCTS-based algorithms utilize neural networks to estimate the value or reward of states. However, hard constraints from previous state abstractions can lead to incorrect aggregation during the early training stage. To reduce the probability of states being mapped to the incorrect abstract space due to bias in network prediction, a novel probability tree state abstraction $ \phi_{Q_{\alpha}^{\psi}}$ is proposed in this work:

\begin{definition}
     (Probability Tree State Abstraction $ \phi_{Q_{\alpha}^{\psi}}$) For a given $\alpha \in [0, 1]$ with node predicate $p_{vM}$ and path predicate $p_{bM}$, the aggregation probability of $\phi_{Q_{\alpha}^{\psi}}$ is defined as:
    % % \begin{small}
    % \begin{equation}
    % \begin{split}
    %     &\mathbb{P}\{\phi_{Q_{\alpha}^{p}}(b_{1})=\phi_{Q_{\alpha}^{p}}(b_{2})\}=\\
    %     &1-\prod^{l}_{v^{b_{1}}_{i} \in b_{1}, v^{b_{2}}_{i} \in b_{2}}{(1-\mathbb{P}\{p_{vM}(v^{b_{1}}_{i},v^{b_{2}}_{i})=1\})} 
    % \end{split}
    % \end{equation}
    % % \end{small}
    % \begin{equation}
    %  \begin{split}
    %     &\mathbb{P}\{p_{vM}(v^{b_{1}}_{i},v^{b_{2}}_{i})=1\} \triangleq \\
    %     &\alpha (1 - D_{JS}(\mathbb{P}\{Q^{p}(v^{b_{1}}_{i},a)\}\|\mathbb{P}\{Q^{p}(v^{b_{2}}_{i},a)\}))
    %     \end{split}
    % \end{equation}
    \begin{equation}
    \begin{split}
        &\mathbb{P}\{p_{bM}(b_{1},b_{2})=1\}\triangleq
        \mathbb{P}\{\phi_{Q_{\alpha}^{\psi}}(b_{1})=\phi_{Q_{\alpha}^{\psi}}(b_{2})\}=
        1-\prod^{l}_{i}{(1-\mathbb{P}\{p_{vM}(v^{b_{1}}_{i},v^{b_{2}}_{i})=1\})}; 
    \end{split}
    \label{psb}
    \end{equation}
    % \end{small}
    \begin{equation}
     \begin{split}
        &\mathbb{P}\{p_{vM}(v^{b_{1}}_{i},v^{b_{2}}_{i})=1\} \triangleq \alpha (1 - D_{JS}(\mathbb{P}\{Q^{\psi}(v^{b_{1}}_{i},a)\}\|\mathbb{P}\{Q^{\psi}(v^{b_{2}}_{i},a)\}))
        \end{split}.
        \label{psn}
    \end{equation}
    where $\mathbb{P}\{Q^{\psi}(v,a)\} = \frac{exp(Q^{*}(v,a))}{\sum_{j \in \mathcal{A}}exp(Q^{*}(v,j))}$, and $D_{JS}$ is the Jensen-Shannon divergence. 
    
    % where $\mathbb{P}\{Q^{p}(v,a)\} = \frac{Q^{*}(v,a)-Q^{min}(v,a)}{\sum_{j \in \mathcal{A}}(Q^{*}(v,j)-Q^{min}(v,a))}$, and $D_{JS}$ is the Jensen-Shannon divergence. 
   
    % The extended $\phi_{Q_{\alpha}^{\psi}}$ is defined as:
    % % \begin{small}
    % \begin{equation}
    % \begin{split}
    %     &\mathbb{P}\{\phi_{Q_{\alpha}^{\psi}}(b_{1})=\phi_{Q_{\alpha}^{\psi}}(b_{2})\}=\\
    %     &1-\prod^{l}_{v^{b_{1}}_{i} \in b_{1}, v^{b_{2}}_{i} \in b_{2}}{(1-\mathbb{P}\{p_{vM}(v^{b_{1}}_{i},v^{b_{2}}_{i})=1\})} 
    % \end{split}
    % \end{equation}
    % % \end{small}
    % \begin{equation}
    %  \begin{split}
    %     &\mathbb{P}\{p_{vM}(v^{b_{1}}_{i},v^{b_{2}}_{i})=1\} \triangleq \\
    %     &\alpha (1 - D_{JS}(\mathbb{P}\{Q^{\psi}(v^{b_{1}}_{i},a)\}\|\mathbb{P}\{Q^{\psi}(v^{b_{2}}_{i},a)\}))
    %     \end{split}
    % \end{equation}
    % where $\mathbb{P}\{Q^{\psi}(v,a)\} = \frac{exp(Q^{*}(v,a))}{\sum_{j \in \mathcal{A}}exp(Q^{*}(v,j))}$.

\end{definition}
$\phi_{Q_{\alpha}^{\psi}}$ encourages nodes that have the same candidate actions with similar value distribution expectations to be aggregated. Using Jensen-Shannon divergence instead of Kullback-Leibler divergence is more advantageous for numerical stability during computation.

\subsection{Integration with Sampled MuZero} \label{sec:algorithm}

Our proposed \textit{PTSA} algorithm can be integrated with state-of-the-art MCTS-based algorithms. The integration includes two main components: offline learning and online searching. Offline learning involves updating the dynamics, prediction, and value networks by sampling trajectories from a buffer. Online searching involves interacting with the environment to obtain high-quality trajectories, similar to MuZero algorithm. Algorithm\ref{alg:TSAES} shows how \textit{PTSA} can be integrated with Sampled MuZero \cite{hubert2021learning} during the searching stage. Compared with the original \textit{Sampled MuZero}, lines 4-12 describe how to collect all searched paths and update the corresponding node values during the multiple iterations. Lines 13-19 describe how
tree state abstraction reduces the search space efficiently. Based on Theorem \ref{transitive}, abstracting the most recently searched path is enough to find the smallest abstract space for MCTS-based methods.  $(\phi(b_i)=\phi(b_s))$ returns a boolean value, where "true" denotes aggregating $b_i$ and $b_s$. This boolean value is determined by calculating the probability $\mathbb{P}(\phi(b_i)=\phi(b_s))$ based on Eq.\eqref{psb} and Eq.\eqref{psn}.

$S_L$ is a list that records the searched paths in the current search tree. $S_L.delete(b)$ and $S_L.add(b)$ refer to removing and recording path $b$ in $S_L$ respectively. The $pruning(b_j)$ action denotes removing unique nodes of path $b_j$ compared to the other abstracted path in the search tree.  $h$ denotes the hidden feature of the original real environment state. Algorithm \ref{alg:TSAES} can be generalized to all tree state abstraction functions by replacing $\phi_{Q^{\psi}_{\alpha}}$. The selection, expansion, and backpropagation steps are the same with \textit{Sampled MuZero} \cite{hubert2021learning}. To maintain the balance of $\sum_{b} N(s, b)$, the visit count of the aggregate node needs to be accumulated into the corresponding state pair. Furthermore, aggregated nodes with different sets of legal actions could lead to unnecessary exploration of invalid or irrelevant parts of the search space, slowing down the search and potentially reducing the quality of the results. In our implementation, we ensure that the aggregated node only expands legal actions for the abstracted state, thus avoiding any negative impacts caused by illegal actions. An undoing aggregation operation has been considered in \ref{alg:TSAES}. As the value estimation becomes more accurate, some previously aggregated nodes in the search will no longer be aggregated in the new search. With each new timestep, the value of the search tree nodes is re-evaluated, leading to changes in the following aggregation results.

In Sampled MuZero algorithm, the computational complexity of simulating from the root node can be expressed as $\mathcal{O}(N_s \cdot (\mathcal{O}(S) + \mathcal{O}(D) + \mathcal{O}(P) + \mathcal{O}(V)))$. $N_s$ represents the number of simulations, and $\mathcal{O}(S)$ denotes the computational complexity of the simulation process, which includes selecting children, expanding, and backpropagating. Additionally, $\mathcal{O}(D)$, $\mathcal{O}(P)$, and $\mathcal{O}(V)$ denote the computational complexities of the dynamics network $\mathbf{d}\theta$, policy network $\mathbf{p}\theta$, and value network $\mathbf{v}_\theta$, respectively. The dynamics network predicts the next hidden state $z$ and corresponding reward $r$ based on the current hidden state $h$ and action $a$. In Algorithm \ref{alg:TSAES}, the time complexity is given by $\mathcal{O}(N_s \cdot (\log_{|\mathcal{A}|} N_s \cdot c_p + \mathcal{O}(S) + \mathcal{O}(D) + \mathcal{O}(P) + \mathcal{O}(V)))$ under balanced search. The balanced search term $\log_{|\mathcal{A}|} N_s \cdot c_p$ accounts for the exploration of child nodes, where $|\mathcal{A}|$ represents the number of possible actions, and $c_p$ is a constant controlling exploration. Since tree state abstraction reduces the branching factor, our algorithm enhances the efficiency of selecting child nodes with a smaller $N_s$. The specific computational time of different methods can be found in Appendix F.

\subsection{Theoretical Justification}\label{sec:Theo}
Following the formulation of the tree state abstractions, the theoretical analysis is conducted from the following perspectives:
i) Transitivity; ii) Aggregation error.

\subsubsection{Transitivity}
% \textbf{Transitivity}

 To abstract tree paths efficiently, our next result shows the relationship between path transitivity and node transitivity:
\newtheorem{theorem1}{Theorem}

\begin{theorem}
\label{transitive}
    For $\forall (v_{1}, v_{2}, v_{3}) \in \mathcal{V}$
     and $(b_{1}, b_{2}, b_{3}) \in \mathcal{B}$: 
     % $\left[\left[p_{bM}\left(b_{1}, b_{2}\right) \wedge p_{bM}\left(b_{2}, b_{3}\right)\right] \Longrightarrow p_{bM}\left(b_{1}, b_{3}\right)]\right] \\
    % \Longleftrightarrow \left[\left[p_{vM}\left(v_{1}, v_{2}\right) \wedge p_{vM}\left(v_{2}, v_{3}\right)\right] \Longrightarrow p_{vM}\left(v_{1}, v_{3}\right)\right]$
\begin{equation}
\begin{aligned}
        &\left[\left[p_{bM}\left(b_{1}, b_{2}\right) \wedge p_{bM}\left(b_{2}, b_{3}\right)\right] \Longrightarrow p_{bM}\left(b_{1}, b_{3}\right)]\right] \Longleftrightarrow \\
        &\left[\left[p_{vM}\left(v_{1}, v_{2}\right) \wedge p_{vM}\left(v_{2}, v_{3}\right)\right] \Longrightarrow p_{vM}\left(v_{1}, v_{3}\right)\right].
\end{aligned}
\end{equation}
\label{the1}
\end{theorem}
\textbf{Proof}. See Appendix A.
% \begin{proof}
%     See Appendix B
% \end{proof}

Theorem \ref{transitive} indicates that path transitivity is a sufficient and necessary condition for node transitivity. Compared with the previous transitive state abstractions, the proposed $ \phi_{Q_{\alpha}^{\psi}}$ is also transitive for paths as the following proposition given:

\begin{proposition}
    The probability of transitivity for $\phi_{Q_{\alpha}^{\psi}}$ can be computed as:
\begin{small}
    \begin{equation}
     \begin{split}
       &\mathbb{P}\{(p_{bM}(b_{1},b_{2}) \wedge p_{bM}(b_{2},b_{3}) \Longrightarrow p_{bM}(b_{1},b_{3}))\} = \\
       &\mathbb{P}\{\phi_{Q_{\alpha}^{\psi}}(b_{1})=\phi_{Q_{\alpha}^{\psi}}(b_{2})\} \mathbb{P}\{\phi_{Q_{\alpha}^{\psi}}(b_{2})=\phi_{Q_{\alpha}^{\psi}}(b_{3})\}
       \mathbb{P}\{\phi_{Q_{\alpha}^{\psi}}(b_{1})=\phi_{Q_{\alpha}^{\psi}}(b_{3})\} + \\
       &(1-\mathbb{P}\{\phi_{Q_{\alpha}^{\psi}}(b_{1})=\phi_{Q_{\alpha}^{\psi}}(b_{3})\})\mathbb{P}\{\phi_{Q_{\alpha}^{\psi}}(b_{2})=\phi_{Q_{\alpha}^{\psi}}(b_{3})\})(1-\mathbb{P}\{\phi_{Q_{\alpha}^{\psi}}(b_{1})=\phi_{Q_{\alpha}^{\psi}}(b_{2})\}.
    \end{split}
\end{equation}
\end{small}

\end{proposition}
\textbf{Proof}. See Appendix B.
% \begin{proof}
%     See Appendix C
% \end{proof}

The computational complexity of computing the smallest possible abstract state space for transitive state abstraction is $\mathcal{O}\left(|\mathcal{S}|^{2} \cdot c_{p}\right)$ \cite{abel2018state}, where $c_{p}$ denotes the computational complexity of evaluating for a given state pair.

\subsubsection{Aggregation Error Bound}
Most approximate state abstractions are non-transitive since their cumulative aggregation value errors are unbounded. We define the aggregation error $\mathrm{E}^{\phi}$ in path set $\mathcal{B}$ as:
\begin{equation}
    \mathrm{E}^{\phi}=\sum_{b_{1},b_{2} \in \mathcal{B}}{|V^{\pi_{\phi}}(b_{1}) - V^{\pi_{\phi}}(b_{2})|}.
\end{equation}
For instance, consider a commonly used approximate state abstraction $\phi _{Q_{\varepsilon}^{*}}$ \cite{abel2016near}: 
\begin{equation}
p_{M}\left(s_{1}, s_{2}\right) \triangleq \max _{a}\left|Q_{M}^{*}\left(s_{1}, a\right)-Q_{M}^{*}\left(s_{2}, a\right)\right| \leq \varepsilon.
\end{equation}
with the value loss $\mathcal{L}_{\phi_{ Q_{\varepsilon}^{*}}} = V^{*}\left(s\right)-V^{\pi_{\phi_{ Q_{\varepsilon}^{*}}}}\left(s \right) \leq \frac{2 \varepsilon R_{max}}{(1-\gamma)^{2}}$.
% \begin{equation}
% \label{valueloss}
%    \mathcal{L}_{\phi_{ Q_{\varepsilon}^{*}}}(s) =  V^{*}\left(s\right)-V^{\pi_{\phi_{ Q_{\varepsilon}^{*}}}}\left(s \right) \leq \frac{2 \varepsilon R_{max}}{(1-\gamma)^{2}}
% \end{equation}
The aggregation error can be bounded as $\frac{6 \varepsilon R_{max}}{(1-\gamma)^{2}}$ if two paths of length 3 are abstracted. 
% Although $\phi _{Q_{\varepsilon}^{*}}$ is non-transitive, Relax-TSA enables $\mathrm{E}^{\phi^{r}_{Q_{\varepsilon}^{*}}}$ to be bounded.

% \begin{proposition}
%     Considering the relaxed transitive state abstraction $\phi^{r}_{ Q_{\varepsilon}^{*}}$ with maximum aggregation times $\tau$, the aggregation error is bounded as:
%     \begin{equation}
%         \mathrm{E}^{\phi^{r}_{Q_{\varepsilon}^{*}}} \leq \frac{2 \varepsilon R_{max}}{(1-\gamma)^{2}}  
%     \end{equation}
% \end{proposition}
% \textbf{Proof}. See Appendix C
% \begin{proof}
% See Appendix D
% \end{proof}

In  Algorithm\ref{alg:TSAES}, the aggregation times must be less than the current branching factor\footnote{The maximum branching factor is less than $|\mathcal{A}|$ and $|\hat{\mathcal{A}}|$ in \textit{MuZero} and \textit{Sampled MuZero} respectively, where $\hat{\mathcal{A}}$ is the sampled action space}. The next result extends to the general tree state abstractions in Algorithm\ref{alg:TSAES}. The path transitivity implies that the abstracted paths present in the current $\mathcal{B_{\phi}}$ must belong to different abstract clusters, which can give the following theorem:

\begin{theorem}
\label{error}
    Considering a general tree state abstraction $\phi$ with a transitive predicate $p (\mathcal{L}_{\phi} \leq \zeta)$, the aggregation error in Alg. \ref{alg:TSAES} under balanced search is bounded as:
    \begin{equation}
        \mathrm{E}^{\phi} < \log_{|\mathcal{A}|}(N_{s}+1)\zeta.
    \end{equation}
    If predicate $p (\mathcal{L}_{\phi} \leq \zeta)$ is not transitive, the aggregation error is bounded as:
    \begin{equation}
        \mathrm{E}^{\phi} < (|\mathcal{A}|-1)\log_{|\mathcal{A}|}(N_{s}+1)\zeta.
    \end{equation}
\end{theorem}
\textbf{Proof.} See Appendix C.

% \begin{proof}
% See Appendix D
% \end{proof}

Theorem \ref{error} provides a theoretical guarantee for finding the smallest $\mathcal{B}_{\phi}$ within aggregation error bound in Algorithm\ref{alg:TSAES}. At the same time, this theorem also points out that the transitivity has an important effect on the tree structure: as the size of the action space $|\mathcal{A}|$ increases, the aggregation error upper bound for transitive tree state abstraction will decrease, while the upper bound for non-transitive state abstraction will increase.

\section{Experiments}
In this section, experiments focus on tree state abstraction for computational efficiency improvement. Firstly, \textit{PTSA} is integrated with state-of-the-art MCTS-based RL algorithms Sampled MuZero and Gumbel MuZero and evaluated on various RL tasks with 10 seeds. The aggregation percentages are also evaluated, which reflects the search space reduction in different tasks. In addition, the comparison between probability tree state abstraction and other state abstraction functions is also conducted.

\subsection{Results on Performance}
To demonstrate the improvement in computational efficiency, the \textit{PTSA} is integrated with Sampled MuZero in Atari and classic control benchmarks, as well as with Gumbel MuZero in the Gomoku benchmark. The performance is compared against several MCTS-based RL algorithms, including (i) MuZero \cite{schrittwieser2020mastering}, (ii) Sampled MuZero \cite{hubert2021learning}, (iii) EfficientZero \cite{ye2021mastering}, (iv) Gumbel MuZero \cite{danihelka2022policy}.

\subsubsection{Atari}
Atari games are visually complex environments that pose challenges for MCTS-based algorithms\cite{schrittwieser2020mastering,ye2022spending}. The results for each benchmark and the normalized result are shown in Figure \ref{fig:Atari}, and the shaded intervals represent the standard deviation of the performance across the different random seeds. As Gumbel MuZero does not require large simulations for Atari and control tasks, we only compare its performance in the Gomoku game. When the simulation times $N=18$, MuZero and Sampled MuZero fail to converge within the maximum training time in some tasks. Since the search space of MCTS is mapped into the abstract space with a smaller branching factor, PTSAZero ($\alpha=0.7$) can converge rapidly with fewer simulations. Although EfficientZero improves the sampling efficiency of MuZero, the increased complexity of the network entails more time to converge with the same computational resources. EfficientZero is tested with different simulation times, and the best case in time efficiency with $N=30$ is shown. The normalized score is computed by $s_{norm} = (s_{agent} - s_{min}) / (s_{max} - s_{min})$, and the normalized time is computed by $t_{norm} = (t_{agent} - t_{min}) / (t_{max} - t_{min})$. The experiment results on Atari benchmarks indicate that Sampled MuZero with \textit{PTSA} can achieve comparable performance with less training time. For MuZero-based algorithms which require massive computational resources and parallel abilities, \textit{PTSA} provides a more efficient method with less computational cost. 

 \begin{figure*}[t!]
\subfigbottomskip=2pt
\centering
\subfigure[Alien] {\includegraphics[width=0.32\linewidth]{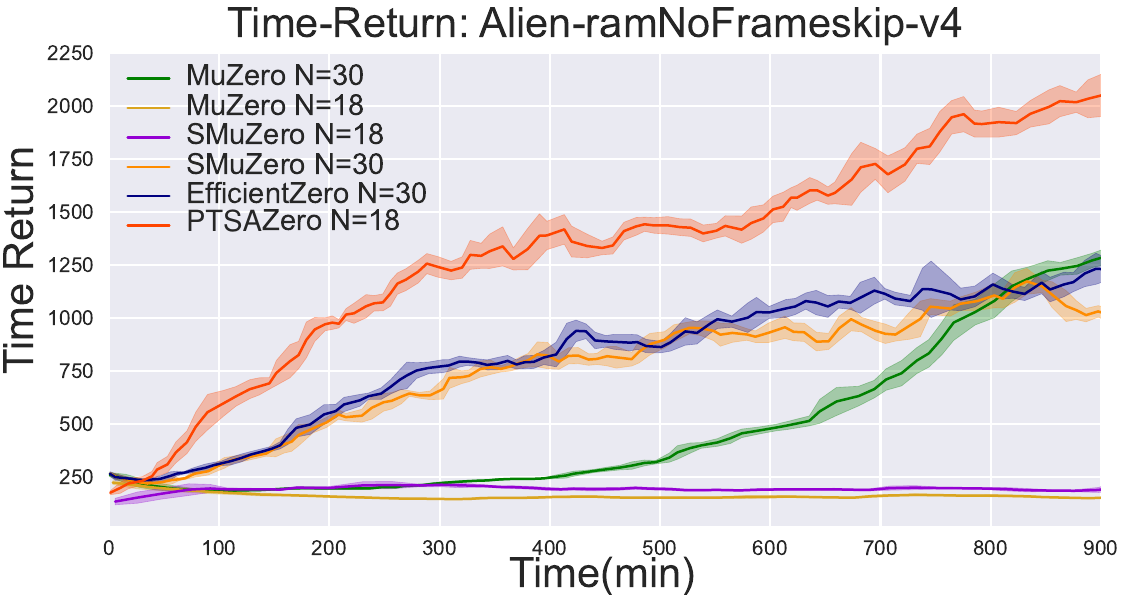}}
\subfigure[Boxing] {\includegraphics[width=0.32\linewidth]{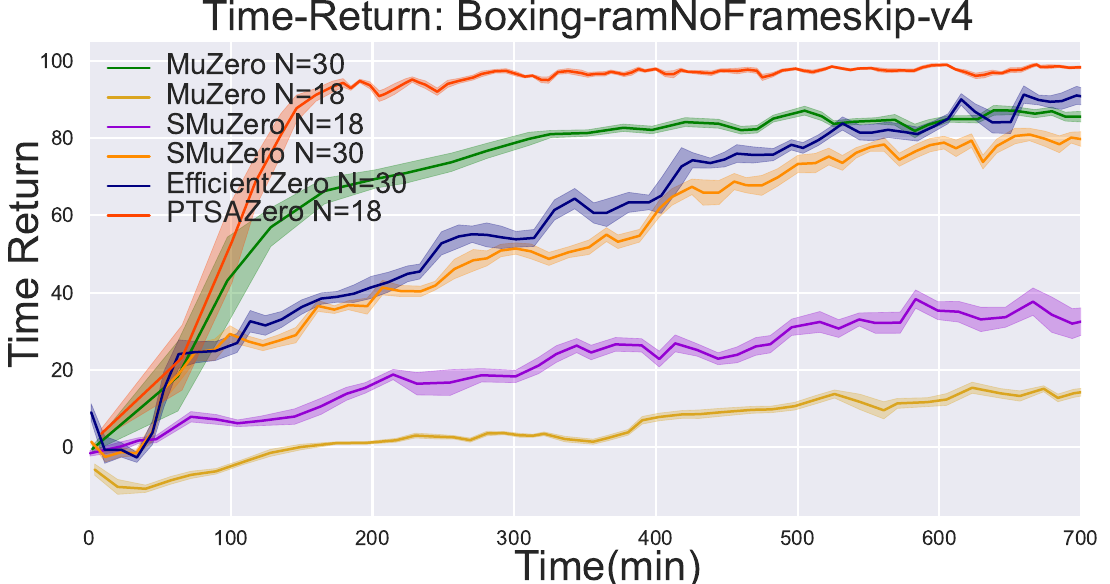}}
\subfigure[Freeway] {\includegraphics[width=0.32\linewidth]{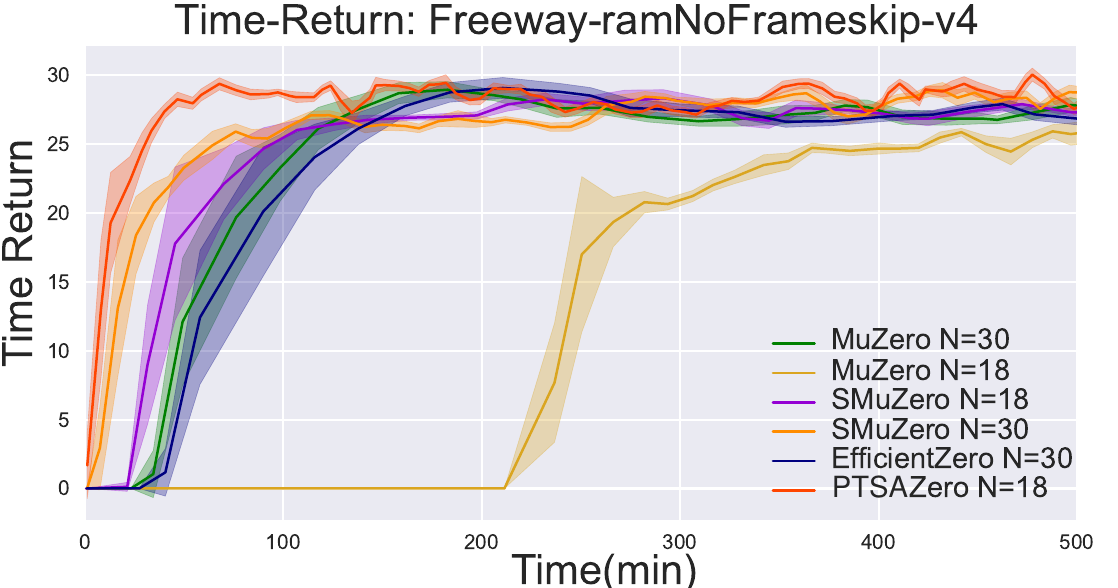}}
\subfigure[Pong] {\includegraphics[width=0.32\linewidth]{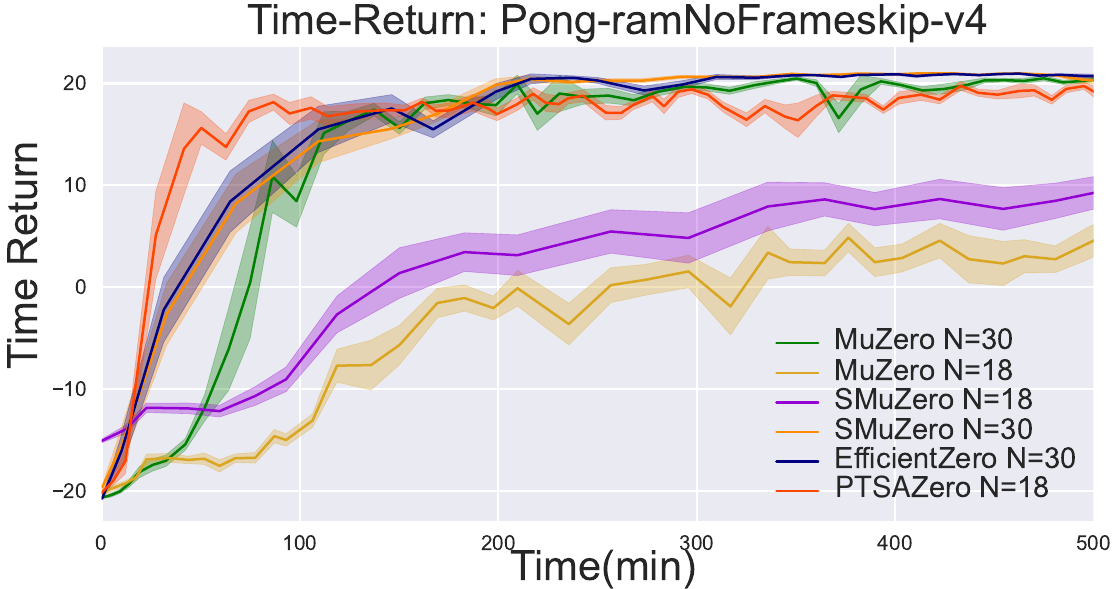}}
\subfigure[Tennis] {\includegraphics[width=0.32\linewidth]{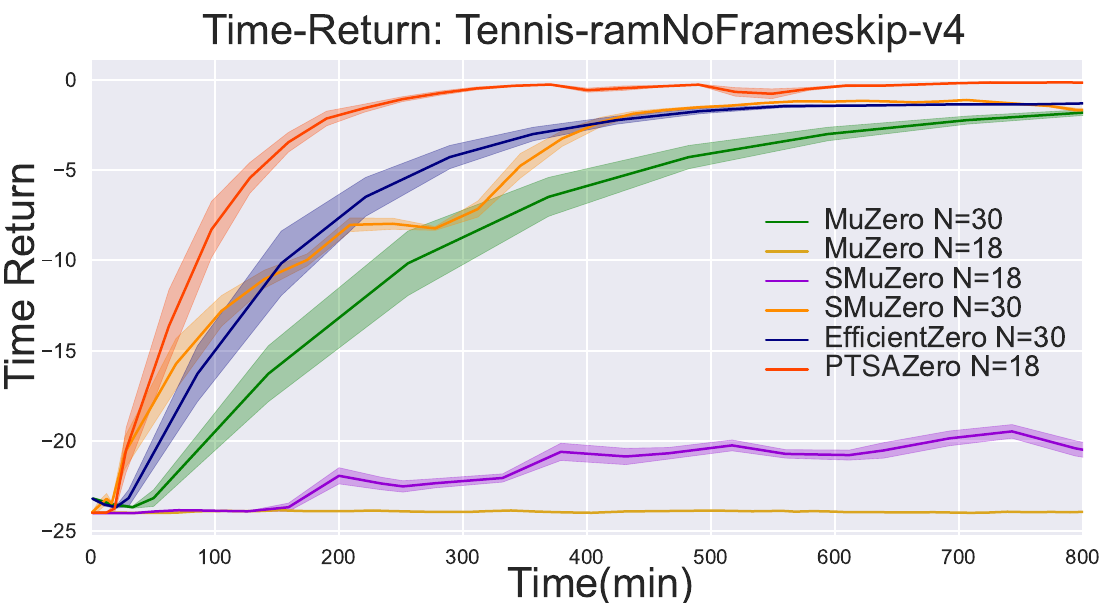}}
\subfigure[Normalized Score] {\includegraphics[width=0.32\linewidth]{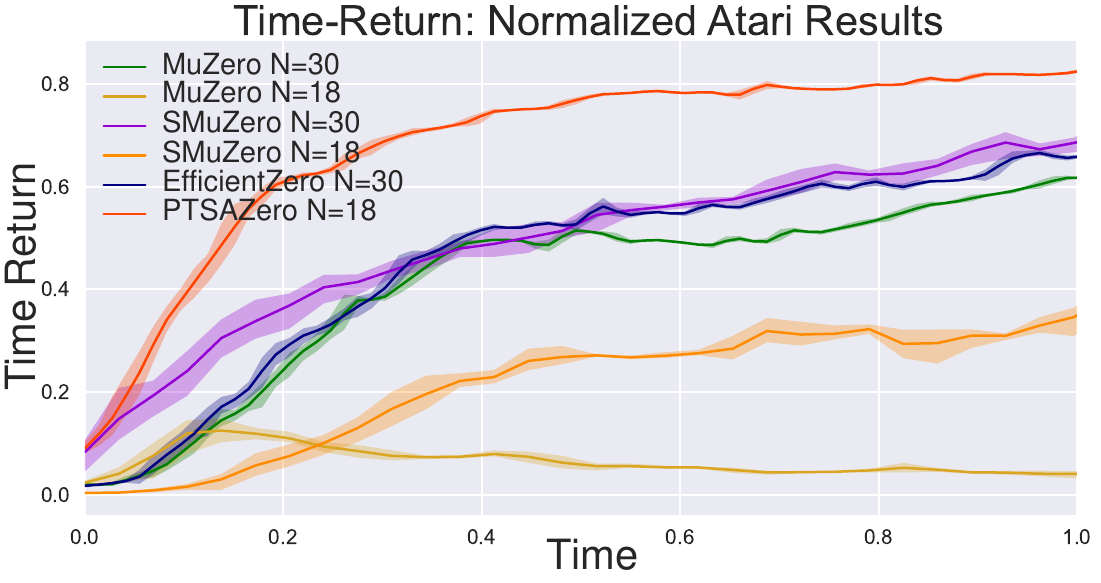}}
\caption{Experiments results on five Atari games (10 seeds) with a normalized score plot. Sampled MuZero with \textit{PTSA} (PTSAZero) is compared with three state-of-the-art MCTS-based methods: MuZero \cite{schrittwieser2020mastering}, Sampled MuZero (SMuZero) \cite{hubert2021learning}, and EfficientZero \cite{ye2021mastering}. The tasks include Alien, Boxing, Freeway, Pong, and Tennis. $N$ denotes the number of simulations. The x-axis is the training time, and the y-axis is the episode return w.r.t training time.} 
\label{fig:Atari}
\end{figure*}

\subsubsection{Control Tasks}
In the comparison experiments conducted on Gym benchmarks, including classic control and box2d tasks, certain modifications are made to increase the tree search space in control tasks. Specifically, the action spaces of CartPole-v0 and LunarLander-v2 are discretized into 100 and 36 dimensions, respectively. This discretization of the action space necessitates MCTS-based algorithms to run a larger number of simulations and increases the number of sampled actions in Sampled MuZero.

For both SMuZero and PTSAZero, the number of sampled actions is set to 25 in CartPole-v0 and 12 in LunarLander-v2. Figure \ref{fig:Gomoku} demonstrates that PTSAZero exhibits superior computational efficiency compared to other methods. In the LunarLander-v2-36 task, PTSAZero significantly improves the training speed of MuZero by a factor of 3.53. These results highlight the effectiveness of PTSAZero in enhancing the efficiency and performance of MuZero-based algorithms in various control tasks.

% \begin{wrapfigure}{R}{0.65\textwidth}
%     \centering
%     \vskip -.5cm
%     \includegraphics[width=0.65\textwidth]{Figure/aggregation_percentage.pdf}
%     \vskip -0.cm
%     \caption{Aggregation percentage on paths during the training process on Atari tasks.}
%     \label{fig:aggregation}
%     \vskip -.15cm
% \end{wrapfigure}

\begin{figure*}[t]
\subfigbottomskip=2pt
\centering
\subfigure[CartPole] {\includegraphics[width=0.23\linewidth]{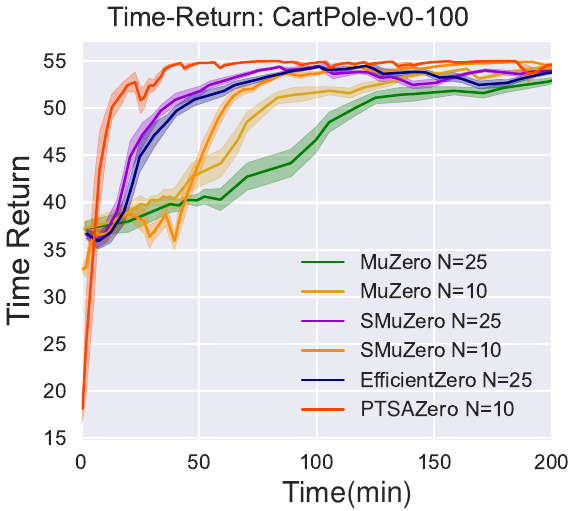}}
\subfigure[Lunarlander] {\includegraphics[width=0.24\linewidth]{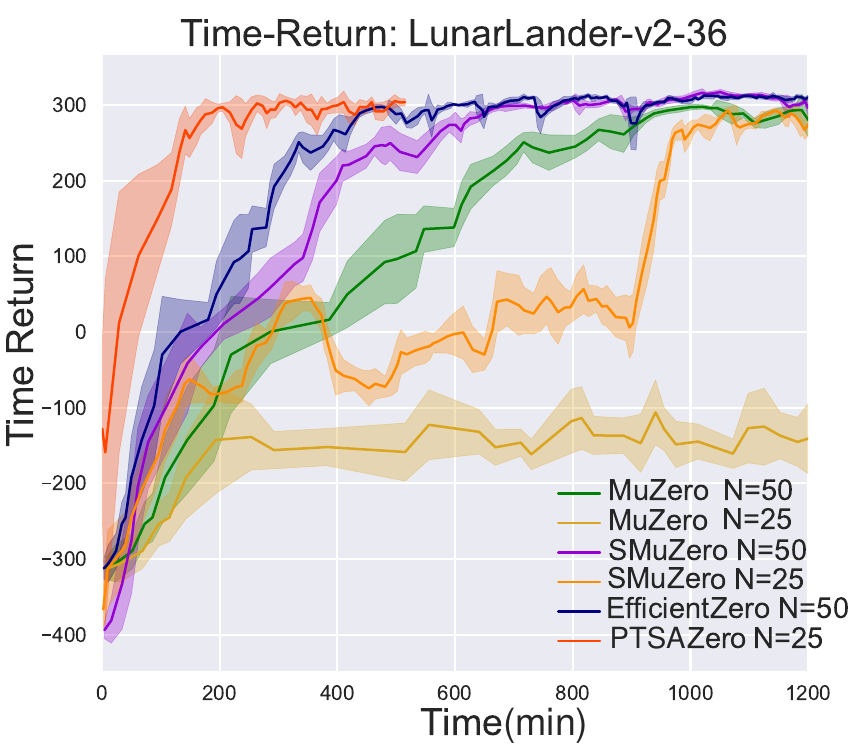}}
\subfigure[Gomoku-15$\times$15] {\includegraphics[width=0.255\linewidth]{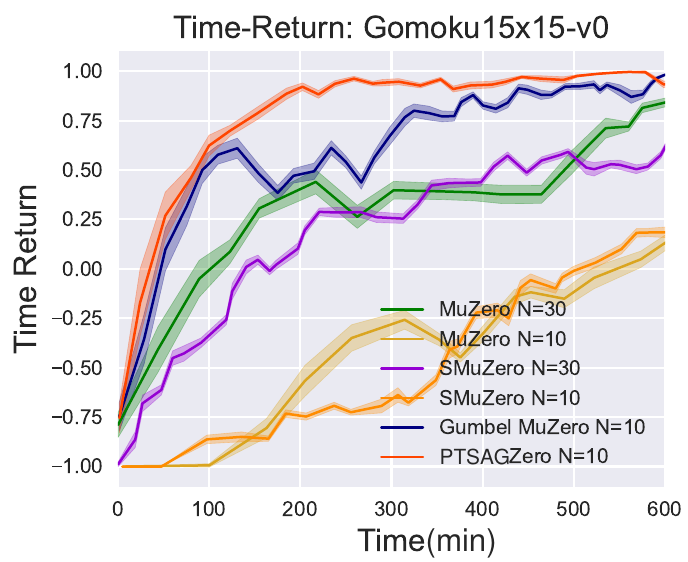}}
\subfigure[Gomoku-19$\times$19] {\includegraphics[width=0.24\linewidth]{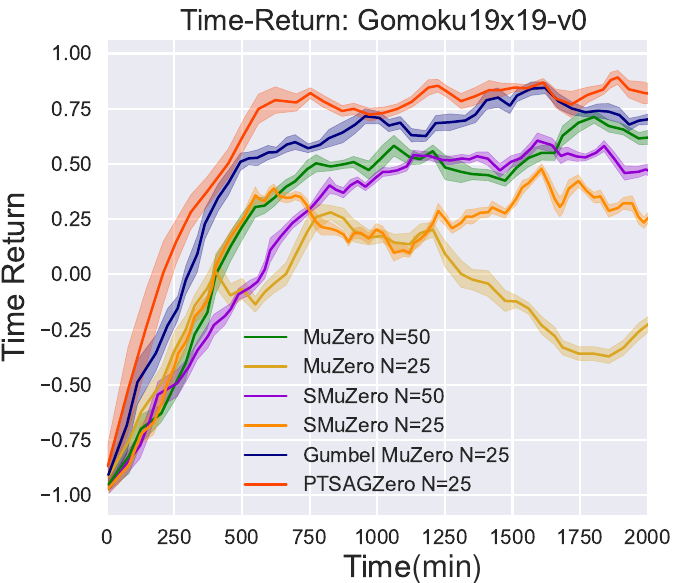}}

\caption{Experiment results of Gym and Gomoku benchmarks (10 seeds). PTSAGZero denotes the Gumbel MuZero with \textit{PTSA} algorithm.} 
\label{fig:Gomoku}
\end{figure*}

\begin{figure*}[t!]
\subfigbottomskip=2pt
\centering
% \subfigure[Atari] {\includegraphics[width=0.32\linewidth]{Figure/aggregation_atari2.pdf}}

\subfigure[Atari] {\includegraphics[width=0.48\linewidth]{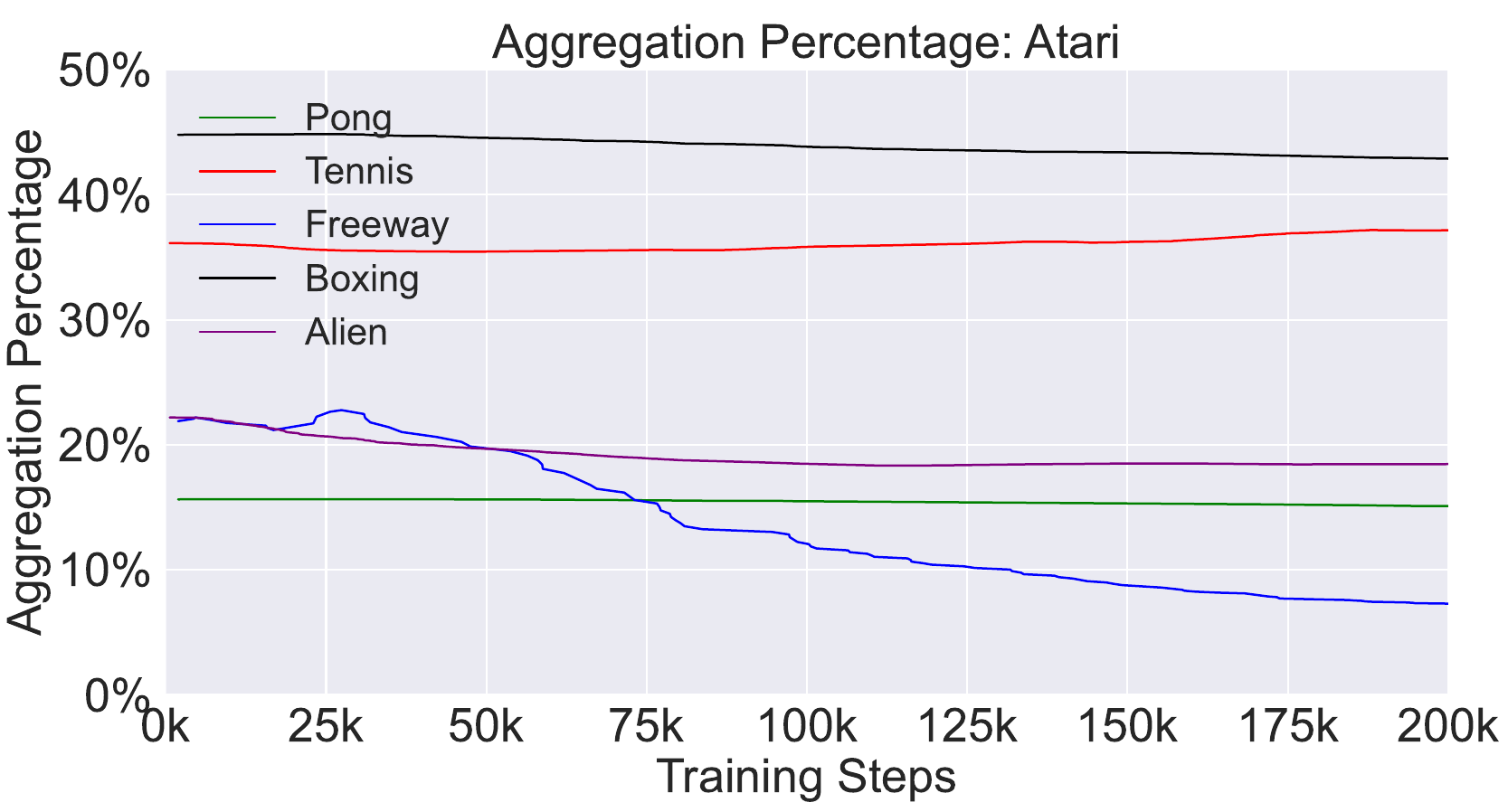}}
\subfigure[Gomoku $\&$ Control] {\includegraphics[width=0.48\linewidth]{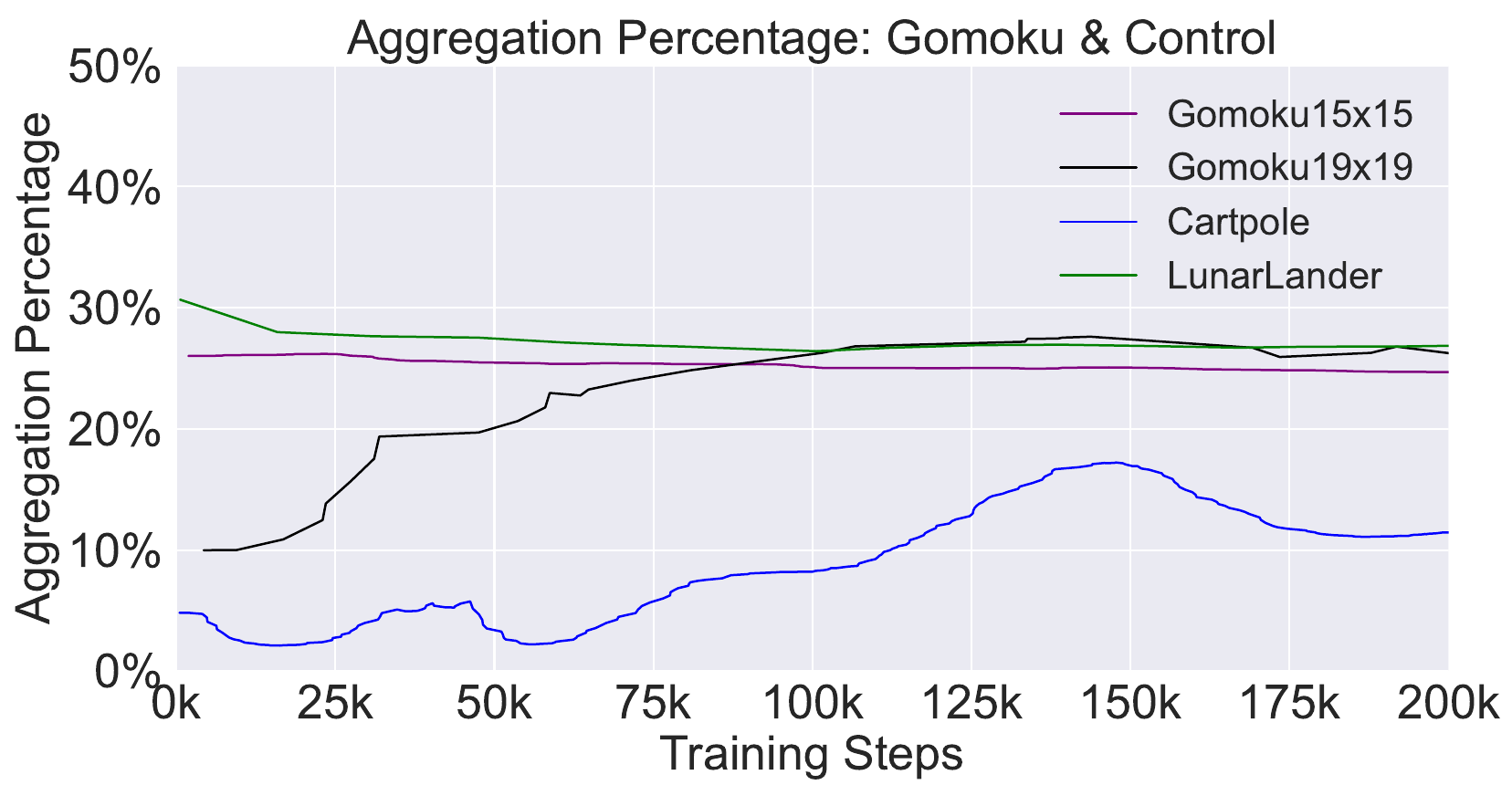}}
\caption{The aggregation percentage on paths during the training process on Atari, Control, and Gomoku tasks varies as the network parameters are updated. } 
\label{fig:aggregation}
\end{figure*}

\subsubsection{Gomoku}
Gomoku is a classic board game with multi-step search, where the agent is asked to beat an expert opponent on the $15 \times 15$ and $19 \times 19$  boards. Gumbel MuZero replaces heuristic mechanisms in original MCTS algorithms for a smaller number of simulations \cite{danihelka2022policy}. In board game tasks, we integrate \textit{PTSA} algorithm with Gumbel MuZero: tree state abstraction $\phi_{Q^{\psi}_{\alpha}}$ is utilized to reduce the search space, as described in Algorithm \ref{alg:TSAES}. The observation of the Gomoku task is a 2-dimensional matrix describing the distribution of the pieces, and the number of sampled actions for SMuZero, Gumbel MuZero, and PTSAGZero is 30. The training returns of different methods w.r.t. training steps are shown in Figure \ref{fig:Gomoku}, where a value of 1 represents a win and -1 represents a loss. The results demonstrate that Gumbel MuZero can speed up the training process of Sampled MuZero, and \textit{PTSA} can provide incremental improvement of \textit{Gumbel MuZero}. In the Gomoku-$19 \times 19$ task, despite accelerating the convergence speed of the original algorithms, \textit{PTSA} does not significantly improve the optimal performance of the algorithms. Experimental results also show that MuZero-based algorithms require a larger number of simulations to learn an effective policy due to the increased size of the action space. 

\subsection{Search Space Reduction} 
To analyze the abstracted tree search space, Figure \ref{fig:aggregation} shows the aggregation percentage (the average proportion of aggregated paths) during the training process of PTSAZero on Atari, Control, and Gomoku tasks. Results demonstrate that tree state abstraction reduces the original branching factors by $10\%$ up to $45\%$. As the network parameters are updated during the training process, the aggregation percentage tends to stabilize. Especially in the Gomoku $19 \times 19$ task, the aggregation percentage increases from $10.1\%$ to $27.8\%$ as the training progresses towards convergence. The converged aggregation percentage may vary depending on specific algorithm parameters and task environments, and it does not imply a certain range of branching factors for all tasks. Each task has unique characteristics and complexities, which can influence the effectiveness of different abstraction functions. It should be emphasized that a larger reduction in the state space does not guarantee improved training efficiency. While a higher aggregation percentage indicates a greater reduction in the search space, aggregation quality of the abstractions also determines the impact on training efficiency. 
Furthermore, since the search space of MCTS is reduced by tree state abstraction, the search depth of PTSAZero is deeper than that of SMuZero with the same number of simulations. More results of search depth can be found in Appendix F.

\begin{table}[t]
    \caption{Speedup comparison of PTSAZero with different state abstraction functions, where Abs denotes the different tree state abstraction functions (the notations are shown in Table \ref{tab: previous state abstractions}). All state abstraction functions are evaluated under the same number of simulations and sampled actions.}
    \label{tab:state abstractions}
\begin{center}
% \begin{small}
% \begin{sc}
\centering
\scalebox{0.85}{
\centering
\begin{tabular}{lcccccccc}
\toprule
Abs &Pong &Boxing &Freeway &Tennis &CartPole  &Lunarlander &Acrobot &Average\\
\midrule 						
$\phi_{a^{*}} $ &{2.35±0.44}&{1.8±0.34}&{2.44±0.46}&{1.59±0.3}&{1.96±0.37}&{3.03±0.57}&{2.11±0.4} & {2.18±0.44}\\
$\phi_{a^{*}}^{\varepsilon}$ &{2.75±0.52}&{1.77±0.33}&{2.31±0.43}&{1.44±0.27}&{2.0±0.38}&{2.93±0.55}&{1.61±0.30} & {2.12±0.53} \\
$\phi_{Q^{*}}$	&{2.29±0.43}& \textbf{2.11±0.4}&{2.47±0.46}&{1.82±0.34}&{1.64±0.31}&{1.83±0.34}&{1.95±0.37} &{2.00±0.26}\\
$\phi_{Q^{*}}^{\varepsilon} $	&{3.01±0.59}&{1.95±0.39}&{2.52±0.49}&\textbf{2.22±0.44}&{1.74±0.35}&{3.13±0.61}&{2.01±0.40}  &{2.37±0.50} \\
$\phi_{Q^{*}_{d}} $	&{2.81±0.55}&{1.84±0.37}&{2.45±0.48}&{2.02±0.4}&\textbf{2.44±0.48}&{3.19±0.62}&{1.91±0.38}  &{2.38±0.46} \\
$\phi_{Q_{\alpha}^{\psi}} $	& \textbf{3.25±0.19}& {2.00±0.21} & \textbf{2.81±0.33} &{1.92±0.17} & {2.14±0.10} & \textbf{3.53±0.31}& \textbf{2.29±0.35} & \textbf{2.56±0.39}\\
\bottomrule
%\vskip -0.5cm
\end{tabular}
}
% \end{small}
\end{center}
\end{table}

\subsection{Comparison with Other State Abstraction Functions}

\begin{wrapfigure}{R}{0.65\textwidth}
    \centering
    % \vskip -.25cm
    \includegraphics[width=0.65\textwidth]{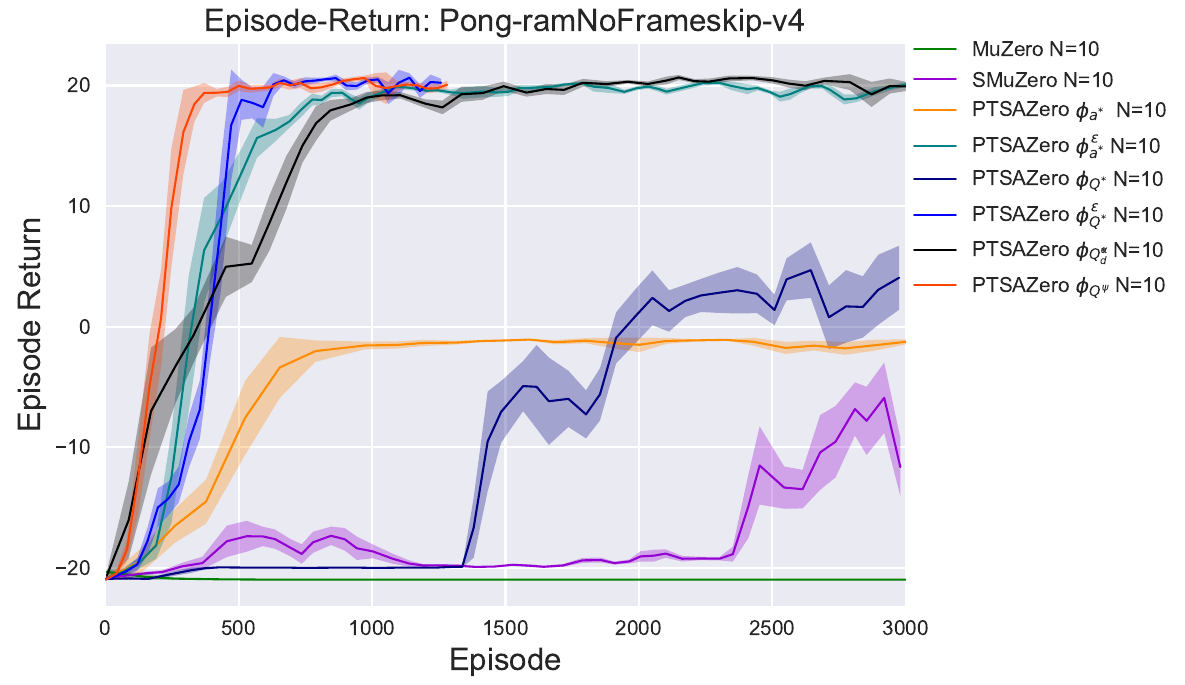}
    % \vskip -0.cm
    \caption{Results of PTSAZero with different state abstraction functions in the Pong task when $N=10$. Some state abstraction functions cannot accurately abstract the search space with a small number of simulations.}
    \label{fig:pong-N=10}
    % \vskip -.15cm
\end{wrapfigure}

We compare the effectiveness of our proposed probability tree state abstraction with other state abstraction functions on Atari and gym benchmarks by integrating them into \textit{PTSA} algorithm. Referring to the speedup evaluation method in \textit{Gumbel MuZero} \cite{danihelka2022policy}, Table \ref{tab:state abstractions} shows the speedup of PTSAZero with different state abstraction functions compared to the \textit{Sampled MuZero}. The specific usage and properties of other state abstraction functions can be found in previous works \cite{abel2016near,abel2018state, li2006towards}. We adjusted hyperparameters for different state abstraction functions and selected the best values ($\epsilon$ and $d$ are set to 0.5 and 0.2, respectively). The results demonstrate that the proposed probability tree state abstraction $\phi_{Q_{\alpha}^{}\psi}$ can achieve better speedup performance on average compared to other state abstraction functions. To evaluate the robustness of $\phi_{Q_{\alpha}^{\psi}}$, PTSAZero with different state abstraction functions with $N=10$ is compared in the Pong task. As shown in Figure \ref{fig:pong-N=10}, SMuZero and MuZero can not learn an effective policy, and PTSAZero with $\phi_{Q^{*}}^{\varepsilon}$ and $\phi_{a^{*}}$ may aggregate incorrect states, which leads to a decrease in performance. We also find that the state abstraction function $\phi_{a^{*}} $ and $\phi_{Q^{*}}$  have strict requirements for the abstract conditions, resulting in a small search space reduction ( $1.5\%$  and $2.8\%$ respectively). Compared with other state abstraction functions, $\phi_{Q_{\alpha}^{\psi}}$ can abstract the original search space more accurately and robustly, thus accelerating the training process.

\section{Conclusion}
This paper introduces \textit{PTSA} algorithm for improving the computational efficiency of MCTS-based algorithms. For efficient abstraction in tree search space, we define path transitivity in the formulation of tree state abstraction. Furthermore, we evaluate that the proposed probability tree state abstraction has a better performance compared with previous state abstraction functions. The experimental results demonstrate that \textit{PTSA} can be integrated with state-of-the-art algorithms and achieve comparable performance with $10\%-45\%$ reduction in tree search space. However, the main limitation of the proposed method is that the parameters of some state abstraction functions need to be manually designed to obtain a more accurate abstract state space. Furthermore, selecting an appropriate state abstraction function based on the characteristics of the state space and transition model is also a potential challenge. Further research will be conducted to address these issues for better performance.

% \bibliographystyle{unsrt}
% \bibliography{main}

\begin{thebibliography}{10}

\bibitem{silver2016mastering}
David Silver, Aja Huang, Chris~J Maddison, Arthur Guez, Laurent Sifre, George Van Den~Driessche, Julian Schrittwieser, Ioannis Antonoglou, Veda Panneershelvam, Marc Lanctot, et~al.
\newblock Mastering the game of go with deep neural networks and tree search.
\newblock {\em nature}, 529(7587):484--489, 2016.

\bibitem{silver2017alphagozero}
David Silver, Julian Schrittwieser, Karen Simonyan, Ioannis Antonoglou, Aja Huang, Arthur Guez, Thomas Hubert, Lucas Baker, Matthew Lai, Adrian Bolton, et~al.
\newblock Mastering the game of go without human knowledge.
\newblock {\em nature}, 550(7676):354--359, 2017.

\bibitem{silver2017mastering}
David Silver, Thomas Hubert, Julian Schrittwieser, Ioannis Antonoglou, Matthew Lai, Arthur Guez, Marc Lanctot, Laurent Sifre, Dharshan Kumaran, Thore Graepel, et~al.
\newblock Mastering chess and shogi by self-play with a general reinforcement learning algorithm.
\newblock {\em arXiv preprint arXiv:1712.01815}, 2017.

\bibitem{schrittwieser2020mastering}
Julian Schrittwieser, Ioannis Antonoglou, Thomas Hubert, Karen Simonyan, Laurent Sifre, Simon Schmitt, Arthur Guez, Edward Lockhart, Demis Hassabis, Thore Graepel, et~al.
\newblock Mastering atari, go, chess and shogi by planning with a learned model.
\newblock {\em Nature}, 588(7839):604--609, 2020.

\bibitem{ye2021mastering}
Weirui Ye, Shaohuai Liu, Thanard Kurutach, Pieter Abbeel, and Yang Gao.
\newblock Mastering atari games with limited data.
\newblock {\em Advances in Neural Information Processing Systems}, 34:25476--25488, 2021.

\bibitem{sun2021research}
Yuxiang Sun, Bo~Yuan, Yongliang Zhang, Wanwen Zheng, Qingfeng Xia, Bojian Tang, and Xianzhong Zhou.
\newblock Research on action strategies and simulations of drl and mcts-based intelligent round game.
\newblock {\em International Journal of Control, Automation and Systems}, 19(9):2984--2998, 2021.

\bibitem{zhao2022mcts}
Debin Zhao, Zhengyuan Hu, and Yinjian Yang.
\newblock An mcts-based recommender system for education complex.
\newblock In {\em 2022 International Conference on Machine Learning and Knowledge Engineering (MLKE)}, pages 323--326. IEEE, 2022.

\bibitem{majeed2019performance}
Sultan~Javed Majeed and Marcus Hutter.
\newblock Performance guarantees for homomorphisms beyond markov decision processes.
\newblock In {\em Proceedings of the AAAI Conference on Artificial Intelligence}, volume~33, pages 7659--7666, 2019.

\bibitem{hutter2016extreme}
Marcus Hutter.
\newblock Extreme state aggregation beyond markov decision processes.
\newblock {\em Theoretical Computer Science}, 650:73--91, 2016.

\bibitem{abel2017toward}
David Abel, Dilip Arumugam, Lucas Lehnert, and Michael~L Littman.
\newblock Toward good abstractions for lifelong learning.
\newblock In {\em Proceedings of the NIPS workshop on hierarchical reinforcement learning}, page~92, 2017.

\bibitem{abel2018state}
David Abel, Dilip Arumugam, Lucas Lehnert, and Michael Littman.
\newblock State abstractions for lifelong reinforcement learning.
\newblock In {\em International Conference on Machine Learning}, pages 10--19. PMLR, 2018.

\bibitem{hostetler2014state}
Jesse Hostetler, Alan Fern, and Tom Dietterich.
\newblock State aggregation in monte carlo tree search.
\newblock In {\em Proceedings of the AAAI Conference on Artificial Intelligence}, volume~28, 2014.

\bibitem{bai2016markovian}
Aijun Bai, Siddharth Srivastava, and Stuart Russell.
\newblock Markovian state and action abstractions for mdps via hierarchical mcts.
\newblock In {\em IJCAI}, pages 3029--3039, 2016.

\bibitem{sokota2021monte}
Samuel Sokota, Caleb~Y Ho, Zaheen Ahmad, and J~Zico Kolter.
\newblock Monte carlo tree search with iteratively refining state abstractions.
\newblock {\em Advances in Neural Information Processing Systems}, 34:18698--18709, 2021.

\bibitem{dockhorn2021game}
Alexander Dockhorn, Jorge Hurtado-Grueso, Dominik Jeurissen, Linjie Xu, and Diego Perez-Liebana.
\newblock Game state and action abstracting monte carlo tree search for general strategy game-playing.
\newblock In {\em 2021 IEEE Conference on Games (CoG)}, pages 1--8. IEEE, 2021.

\bibitem{even2003approximate}
Eyal Even-Dar and Yishay Mansour.
\newblock Approximate equivalence of markov decision processes.
\newblock In {\em Learning Theory and Kernel Machines}, pages 581--594. Springer, 2003.

\bibitem{sutton2018reinforcement}
Richard~S Sutton and Andrew~G Barto.
\newblock {\em Reinforcement learning: An introduction}.
\newblock MIT press, 2018.

\bibitem{browne2012survey}
Cameron~B Browne, Edward Powley, Daniel Whitehouse, Simon~M Lucas, Peter~I Cowling, Philipp Rohlfshagen, Stephen Tavener, Diego Perez, Spyridon Samothrakis, and Simon Colton.
\newblock A survey of monte carlo tree search methods.
\newblock {\em IEEE Transactions on Computational Intelligence and AI in games}, 4(1):1--43, 2012.

\bibitem{chen2020driving}
Jienan Chen, Cong Zhang, Jinting Luo, Junfei Xie, and Yan Wan.
\newblock Driving maneuvers prediction based autonomous driving control by deep monte carlo tree search.
\newblock {\em IEEE transactions on vehicular technology}, 69(7):7146--7158, 2020.

\bibitem{schrittwieser2021online}
Julian Schrittwieser, Thomas Hubert, Amol Mandhane, Mohammadamin Barekatain, Ioannis Antonoglou, and David Silver.
\newblock Online and offline reinforcement learning by planning with a learned model.
\newblock {\em Advances in Neural Information Processing Systems}, 34:27580--27591, 2021.

\bibitem{hubert2021learning}
Thomas Hubert, Julian Schrittwieser, Ioannis Antonoglou, Mohammadamin Barekatain, Simon Schmitt, and David Silver.
\newblock Learning and planning in complex action spaces.
\newblock In {\em International Conference on Machine Learning}, pages 4476--4486. PMLR, 2021.

\bibitem{danihelka2022policy}
Ivo Danihelka, Arthur Guez, Julian Schrittwieser, and David Silver.
\newblock Policy improvement by planning with gumbel.
\newblock In {\em International Conference on Learning Representations}, 2022.

\bibitem{abel2016near}
David Abel, David Hershkowitz, and Michael Littman.
\newblock Near optimal behavior via approximate state abstraction.
\newblock In {\em International Conference on Machine Learning}, pages 2915--2923. PMLR, 2016.

\bibitem{gelly2006exploration}
Sylvain Gelly and Yizao Wang.
\newblock Exploration exploitation in go: Uct for monte-carlo go.
\newblock In {\em NIPS: Neural Information Processing Systems Conference On-line trading of Exploration and Exploitation Workshop}, 2006.

\bibitem{rosin2011multi}
Christopher~D Rosin.
\newblock Multi-armed bandits with episode context.
\newblock {\em Annals of Mathematics and Artificial Intelligence}, 61(3):203--230, 2011.

\bibitem{andre2002state}
David Andre and Stuart~J Russell.
\newblock State abstraction for programmable reinforcement learning agents.
\newblock In {\em Aaai/iaai}, pages 119--125, 2002.

\bibitem{anand2016oga}
Ankit Anand, Ritesh Noothigattu, Parag Singla, et~al.
\newblock Oga-uct: On-the-go abstractions in uct.
\newblock In {\em Twenty-Sixth International Conference on Automated Planning and Scheduling}, 2016.

\bibitem{lan2022alphazero}
Li-Cheng Lan, Huan Zhang, Ti-Rong Wu, Meng-Yu Tsai, I~Wu, Cho-Jui Hsieh, et~al.
\newblock Are alphazero-like agents robust to adversarial perturbations?
\newblock {\em Advances in Neural Information Processing Systems}, 35:11229--11240, 2022.

\bibitem{li2006towards}
Lihong Li, Thomas~J Walsh, and Michael~L Littman.
\newblock Towards a unified theory of state abstraction for mdps.
\newblock In {\em AI\&M}, 2006.

\bibitem{ye2022spending}
Weirui Ye, Pieter Abbeel, and Yang Gao.
\newblock Spending thinking time wisely: Accelerating mcts with virtual expansions.
\newblock {\em arXiv preprint arXiv:2210.12628}, 2022.

\bibitem{moritz2018ray}
Philipp Moritz, Robert Nishihara, Stephanie Wang, Alexey Tumanov, Richard Liaw, Eric Liang, Melih Elibol, Zongheng Yang, William Paul, Michael~I Jordan, et~al.
\newblock Ray: A distributed framework for emerging $\{$AI$\}$ applications.
\newblock In {\em 13th USENIX Symposium on Operating Systems Design and Implementation (OSDI 18)}, pages 561--577, 2018.

\end{thebibliography}

\newpage
\appendix

\section{Path and Node Transitivity}
\begin{theorem}
    For $\forall (v_{1}, v_{2}, v_{3}) \in \mathcal{V}$
    and $(b_{1}, b_{2}, b_{3}) \in \mathcal{B}$:
\begin{equation}
\begin{aligned}
        &\left[\left[p_{bM}\left(b_{1}, b_{2}\right) \wedge p_{bM}\left(b_{2}, b_{3}\right)\right] \Longrightarrow p_{bM}\left(b_{1}, b_{3}\right)]\right] \Longleftrightarrow \\
        &\left[\left[p_{vM}\left(v_{1}, v_{2}\right) \wedge p_{vM}\left(v_{2}, v_{3}\right)\right] \Longrightarrow p_{vM}\left(v_{1}, v_{3}\right)\right]
\end{aligned}
\end{equation}
\end{theorem}
\begin{proof}
Consider three paths only contain one node respectively: $b_{1}=\{v_{1}\}, b_{2}=\{v_{2}\}, b_{3}=\{v_{3}\}$. 
For $(b_{1}, b_{2}, b_{3}) \in \mathcal{B}$:
\begin{equation}
    p_{bM}\left(b_{1}, b_{2}\right) = p_{vM}\left(v_{1}, v_{2}\right)
\end{equation}
\begin{equation}
    p_{bM}\left(b_{2}, b_{3}\right) = p_{vM}\left(v_{2}, v_{3}\right)
\end{equation}
\begin{equation}
    p_{bM}\left(b_{1}, b_{3}\right) = p_{vM}\left(v_{1}, v_{3}\right)
\end{equation}
 If the condition is reversed, the equation can also hold. Consider three arbitrary branches (sibling branches of common nodes are omitted): $b_{1}=\{v_{1},v_{2}\}, b_{2}=\{v_{3},v_{4}\}, b_{3}=\{v_{5},v_{6}\}$.

s.t. $p_{vM}\left(v_{1}, v_{2}\right) \wedge p_{vM}\left(v_{2}, v_{3}\right)\Longrightarrow p_{vM}\left(v_{1}, v_{3}\right)$

According to the definition of branch predicate:
\begin{equation}
    p_{bM}\left(b_{1}, b_{2}\right) = p_{vM}\left(v_{1}, v_{3}\right) \wedge p_{vM}\left(v_{2}, v_{4}\right) 
\end{equation}

\begin{equation}
    p_{bM}\left(b_{2}, b_{3}\right) = p_{vM}\left(v_{3}, v_{5}\right) \wedge p_{vM}\left(v_{4}, v_{6}\right)
\end{equation}

\begin{equation}
    \begin{aligned}
    &p_{bM}\left(b_{1}, b_{2}\right) \wedge p_{bM}\left(b_{2}, b_{3}\right) \\
    &= p_{vM}\left(v_{1}, v_{3}\right) \wedge p_{vM}\left(v_{2}, v_{4}\right) \wedge p_{vM}\left(v_{3}, v_{5}\right) \wedge p_{vM}\left(v_{4}, v_{6}\right) \\
    &= p_{vM}\left(v_{1}, v_{5}\right) \wedge p_{vM}\left(v_{2}, v_{6}\right) \\
    &= p_{bM}\left(b_{1}, b_{3}\right)
\end{aligned}
\end{equation}
\end{proof}

\section{Probability of Transitivity}

\begin{proposition}
    The probability of transitivity for $\phi_{Q_{\alpha}^{\psi}}$ can be computed as:
\begin{equation}
     \begin{split}
       &\mathbb{P}\{(p_{bM}(b_{1},b_{2}) \wedge p_{bM}(b_{2},b_{3}) \Longrightarrow p_{bM}(b_{1},b_{3}))\} = \\
       &\mathbb{P}\{\phi_{Q_{\alpha}^{\psi}}(b_{1})=\phi_{Q_{\alpha}^{\psi}}(b_{2})\} \mathbb{P}\{\phi_{Q_{\alpha}^{\psi}}(b_{2})=\phi_{Q_{\alpha}^{\psi}}(b_{3})\}   \\
       &\mathbb{P}\{\phi_{Q_{\alpha}^{\psi}}(b_{1})=\phi_{Q_{\alpha}^{\psi}}(b_{3})\} + (1-\mathbb{P}\{\phi_{Q_{\alpha}^{\psi}}(b_{1})=\phi_{Q_{\alpha}^{\psi}}(b_{3})\}) \\
       & (1-\mathbb{P}\{\phi_{Q_{\alpha}^{\psi}}(b_{1})=\phi_{Q_{\alpha}^{\psi}}(b_{2})\} \mathbb{P}\{\phi_{Q_{\alpha}^{\psi}}(b_{2})=\phi_{Q_{\alpha}^{\psi}}(b_{3})\})
    \end{split}
\end{equation}
\end{proposition}
\begin{proof}
    All cases can be divided into two categories:
    \begin{itemize}
        \item $p_{bM}(b_{1},b_{3}) = 1$
        \item $p_{bM}(b_{1},b_{3}) = 0$
    \end{itemize}
    If  $p_{bM}(b_{1},b_{3}) = 1$, $p_{bM}(b_{1},b_{2}) = 1$ and $p_{bM}(b_{2},b_{3}) = 1$.
    \begin{equation}
    \begin{split}
        &\mathbb{P}_{1} = \mathbb{P}\{\phi_{Q_{\alpha}^{\psi}}(b_{1})=\phi_{Q_{\alpha}^{\psi}}(b_{2})\} \mathbb{P}\{\phi_{Q_{\alpha}^{\psi}}(b_{2})=\phi_{Q_{\alpha}^{\psi}}(b_{3})\}   \\
       &\mathbb{P}\{\phi_{Q_{\alpha}^{\psi}}(b_{1})=\phi_{Q_{\alpha}^{\psi}}(b_{3})\}
    \end{split}
    \end{equation}
    If $p_{bM}(b_{1},b_{3}) = 0$, $p_{bM}(b_{1},b_{2}) = 0$ or $p_{bM}(b_{2},b_{3}) = 0$.
    \begin{equation}
    \begin{split}
        &\mathbb{P}_{2} = (1-\mathbb{P}\{\phi_{Q_{\alpha}^{\psi}}(b_{1})=\phi_{Q_{\alpha}^{\psi}}(b_{3})\}) \\
       &(1-\mathbb{P}\{\phi_{Q_{\alpha}^{\psi}}(b_{1})=\phi_{Q_{\alpha}^{\psi}}(b_{2})\} \mathbb{P}\{\phi_{Q_{\alpha}^{\psi}}(b_{2})=\phi_{Q_{\alpha}^{\psi}}(b_{3})\})
    \end{split}
    \end{equation}
    \begin{equation}
    \begin{split}
        &\mathbb{P}\{(p_{bM}(b_{1},b_{2}) \wedge p_{bM}(b_{2},b_{3}) \Longrightarrow p_{bM}(b_{1},b_{3}))\} = \mathbb{P}_{1}+\mathbb{P}_{2}=\\
        &\mathbb{P}\{\phi_{Q_{\alpha}^{\psi}}(b_{1})=\phi_{Q_{\alpha}^{\psi}}(b_{2})\} \mathbb{P}\{\phi_{Q_{\alpha}^{\psi}}(b_{2})=\phi_{Q_{\alpha}^{\psi}}(b_{3})\}   \\
       &\mathbb{P}\{\phi_{Q_{\alpha}^{\psi}}(b_{1})=\phi_{Q_{\alpha}^{\psi}}(b_{3})\} + (1-\mathbb{P}\{\phi_{Q_{\alpha}^{\psi}}(b_{1})=\phi_{Q_{\alpha}^{\psi}}(b_{3})\}) \\
       & (1-\mathbb{P}\{\phi_{Q_{\alpha}^{\psi}}(b_{1})=\phi_{Q_{\alpha}^{\psi}}(b_{2})\} \mathbb{P}\{\phi_{Q_{\alpha}^{\psi}}(b_{2})=\phi_{Q_{\alpha}^{\psi}}(b_{3})\})
    \end{split}
    \end{equation}
\end{proof}

\section{Aggregation Error Bound of PTSA}

\begin{theorem}
    Considering a general tree state abstraction $\phi$ with a transitive predicate $p (\mathcal{L}_{\phi} \leq \zeta)$, the aggregation error in Alg. \ref{alg:TSAES} under balanced search is bounded as:
    \begin{equation}
        \mathrm{E}^{\phi} < \log_{|\mathcal{A}|}(N_{s}+1)\zeta
    \end{equation}
    If predicate $p (\mathcal{L}_{\phi} \leq \zeta)$ is not transitive, the aggregation error is bounded as:
    \begin{equation}
        \mathrm{E}^{\phi} < (|\mathcal{A}|-1)\log_{|\mathcal{A}|}(N_{s}+1)\zeta
    \end{equation}
\end{theorem}

% \begin{figure}[b]
%     \centering
%     \includegraphics[width=0.85\textwidth]{Figure/AbS_parallelism.png}
%     \caption{Parallelism framework of PTSAZero implementation.}
%     \label{fig:parallelism}
% \end{figure}

\begin{proof}\let\qed\relax
    Assuming an action space of size $A$ and expansion of one child node per simulation, the search tree under balanced search in \textit{MuZero} algorithm can be viewed as an $A$-ary tree. The average depth of the tree can be approximated as:$$ D \approx \log_{A} (N+1) $$, where $(N+1)$ represents the total number of nodes in the search tree, with $+1$ compensating for the root node that is not included in the depth calculation.
    
    Considering transitivity among all searched paths, it is possible to aggregate at most two paths, resulting in a maximum aggregation error equal to the cumulative error of all nodes on these two paths:
    \begin{equation}
        \mathrm{E}^{\phi_{max}} \leq \log_{|\mathcal{A}|}(N_{s}+1)\zeta
    \end{equation}
    \begin{equation}
        \mathrm{E}^{\phi^{r}} < \mathrm{E}^{\phi^{r}_{max}} \leq \log_{|\mathcal{A}|}(N_{s}+1)\zeta
    \end{equation}
    Considering non-transitivity among all searched paths, all paths should be considered for aggregation, the maximum number of subtrees under the root node in MuZero algorithm is limited by $\mathcal{|A|}$. Therefore, the maximum aggregation error after merging is determined by the cumulative error of all nodes in the largest subtrees under the root node:
    \begin{equation}
        \mathrm{E}^{\phi^{r}} < \mathrm{E}^{\phi^{r}_{max}} \leq (|\mathcal{A}|-1)\log_{|\mathcal{A}|}(N_{s}+1)\zeta
    \end{equation}
    
\end{proof}

\begin{figure}[t]
\subfigbottomskip=2pt
\centering
\subfigure[Parallelism Framework] {\includegraphics[width=0.45\linewidth]{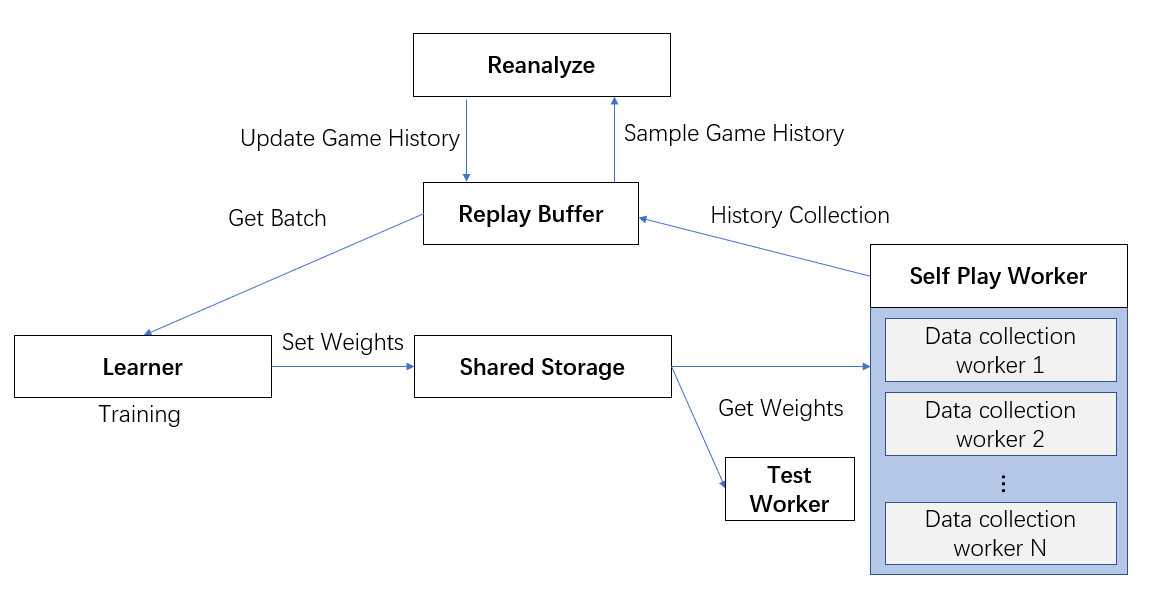}}
\subfigure[Algorithm Flow] {\includegraphics[width=0.4\linewidth]{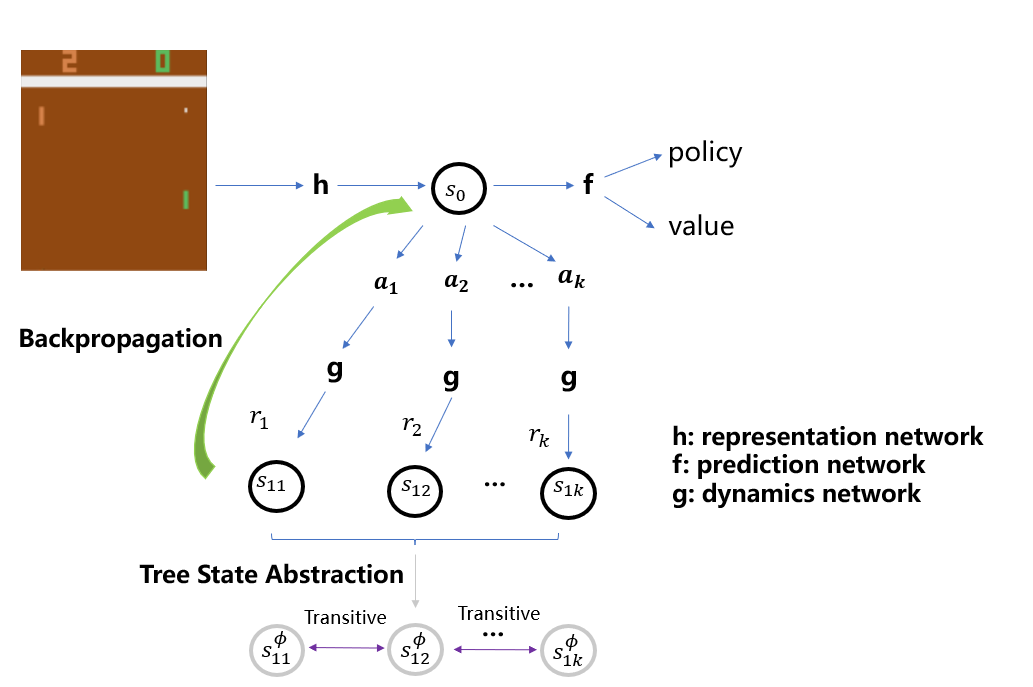}}
\caption{Parallelism framework of PTSAZero implementation and PTSAZero algorithm flow.}
\label{fig:parallelism}
\end{figure}

% \begin{figure}[t]
%     \centering
%     \includegraphics[width=0.85\textwidth]{Figure/AbSZero_framework.png}
%     \caption{PTSAZero algorithm flow.}
%     \label{fig:framework}
% \end{figure}

\begin{figure}[t]
\subfigbottomskip=2pt
\centering
\subfigure[Pong] {\includegraphics[width=0.45\linewidth]{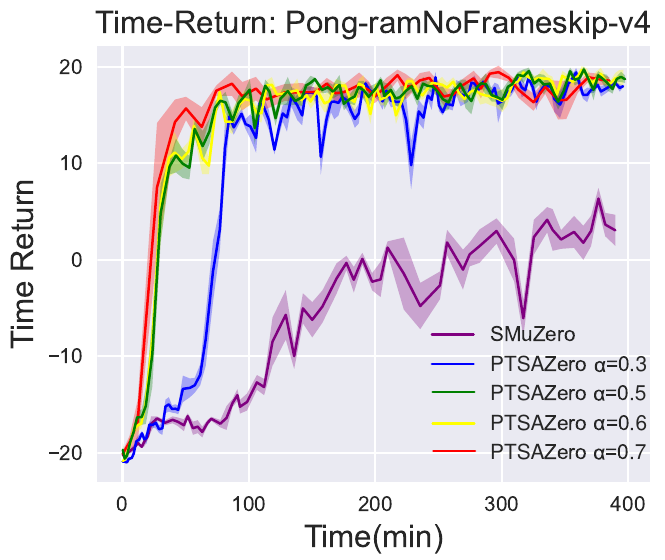}}
\subfigure[Freeway] {\includegraphics[width=0.48\linewidth]{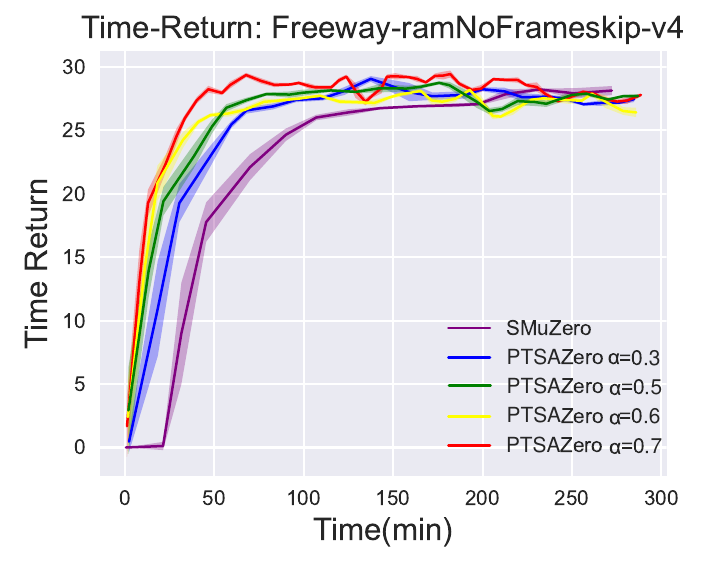}}
\caption{Experiment results of different parameters in probability state abstraction on Atari benchmarks.}
\label{fig:ablation}
\end{figure}

\begin{figure}[t]
    \centering
    % \vskip -.5cm
    \includegraphics[width=0.65\textwidth]{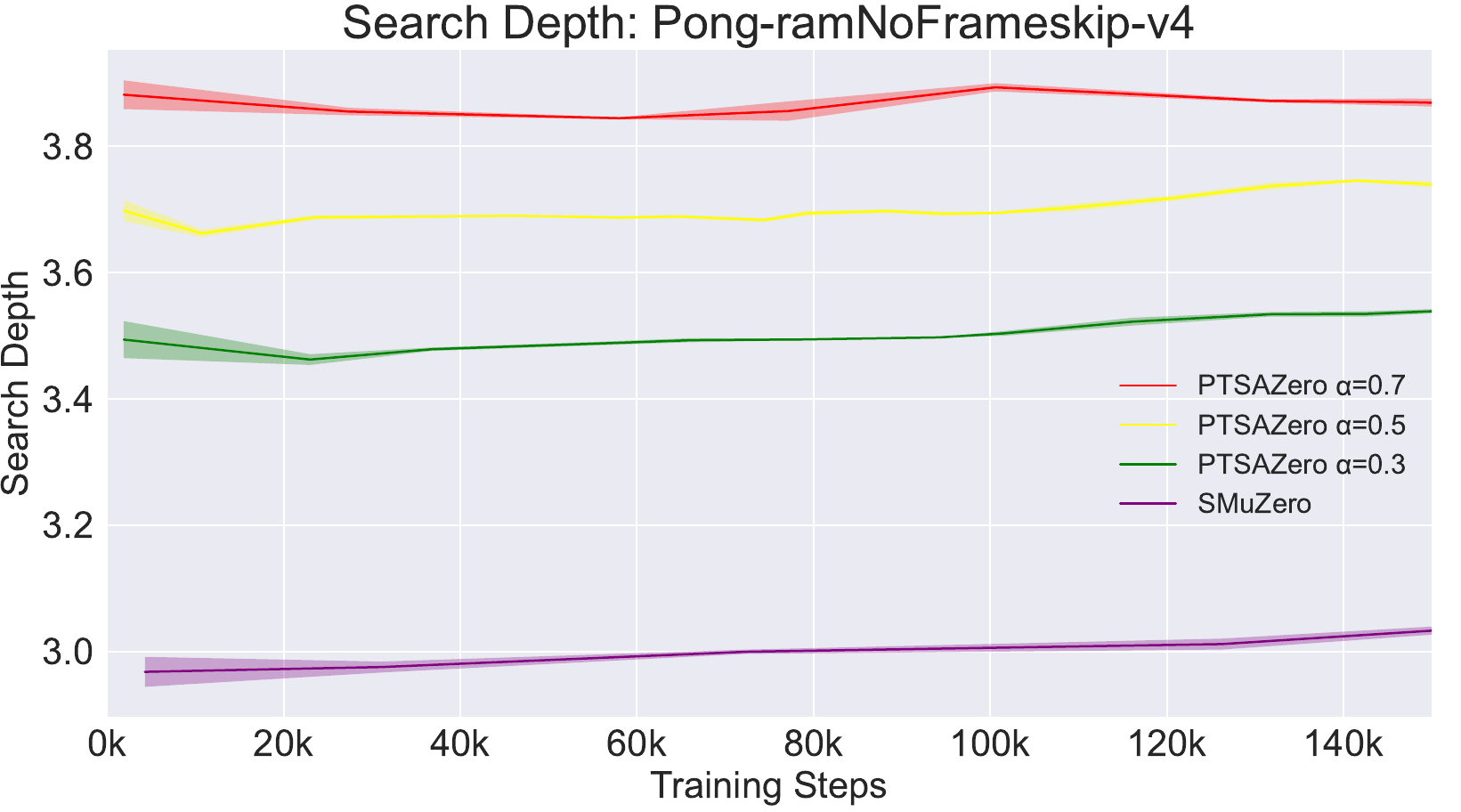}
    \vskip -0.25cm
    \caption{Comparison of average search depth in Pong. The average search depth represents the average path length of the search tree.}
    \label{fig:search depth}
    % \vskip -.5cm
\end{figure}

\section{Implementation}
%%%%%%%%%%%%%%%%%%%%%%%%%%%%%%%%%%%%%%%%%%%%%%%%%%%%%%%%%%%%%%%%%%%%%%%%%%%%%%%
%%%%%%%%%%%%%%%%%%%%%%%%%%%%%%%%%%%%%%%%%%%%%%%%%%%%%%%%%%%%%%%%%%%%%%%%%%%%%%%

All experiments are run on Intel Xeon ICX Platinum 8358 and GeForce RTX 3090. The implementation of MuZero is based on the code from \textbf{muzero-general} (https://github.com/werner-duvaud/muzero-general) and \textbf{model-based-rl} (https://github.com/JimOhman/\\
model-based-rl). The modification of SMuZero has three improvements over MuZero:
\begin{itemize}
    \item When expanding nodes, the MCTS only considers a set of sampled actions from the original action space, instead of enumerating all actions. The proposal distribution $\beta_{p}(a|s)$ is based on the policy network, which is consistent with \cite{hubert2021learning}.
    \item The UCB formula does not use the raw prior $\pi$, but instead the sample-based equivalent $\frac{\hat{\pi}}{\pi}$.
    \item Instead of utilizing the distribution of all actions, the policy is updated on the sampled actions.
\end{itemize}
The implementation of EfficientZero is based on the code from EfficientZero (https://github.com/YeWR/
EfficientZero). The network structures of all methods are modified as SMuZerO \cite{hubert2021learning} in Atari benchmarks for a fair comparison. 

The parallelism implementations of all methods are based on \textbf{ray} library \cite{moritz2018ray}. The parallelism framework and algorithm flow of PTSAZero are shown in Figure \ref{fig:parallelism} for better reproduction.

\begin{table}[t]
\vspace{-0cm}
  \centering
    \begin{tabular}{cc}
    \hline
    Parameter & Setting\\
    \hline
    \textit{frame size}	& $96 \times 96$ \\
    \textit{number of actors}	& 7 \\
    \textit{max history length} & 500 \\
    \textit{visit softmax temperatures} & 1.0,0.5,0.25	\\
    \textit{root dirichlet alpha} &0.25	\\
    \textit{root exploration fraction} &0.25  \\
    \textit{pb c base} &19652	\\
    \textit{pb c init} &1.25	\\
    \textit{buffer size} &10000	\\
    \textit{batch size} &256	\\
    \textit{td steps} &50	\\
    \textit{num unroll steps} &5	\\
    \textit{send weights frequency} &500 \\
    \textit{weight sync frequency} &1000	\\
    \textit{discount} &0.997	\\
    \textit{optimizer} &AdamW	\\
    \textit{lr init} &0.0008	\\
    \hline
    \end{tabular}
   \caption{Specific parameters in Atari benchmark.} 
  \label{tab:Atari Parameter}%
\end{table}%

\begin{figure*}[t!]
\subfigbottomskip=2pt
\centering
\subfigure[Assault] {\includegraphics[width=0.45\linewidth]{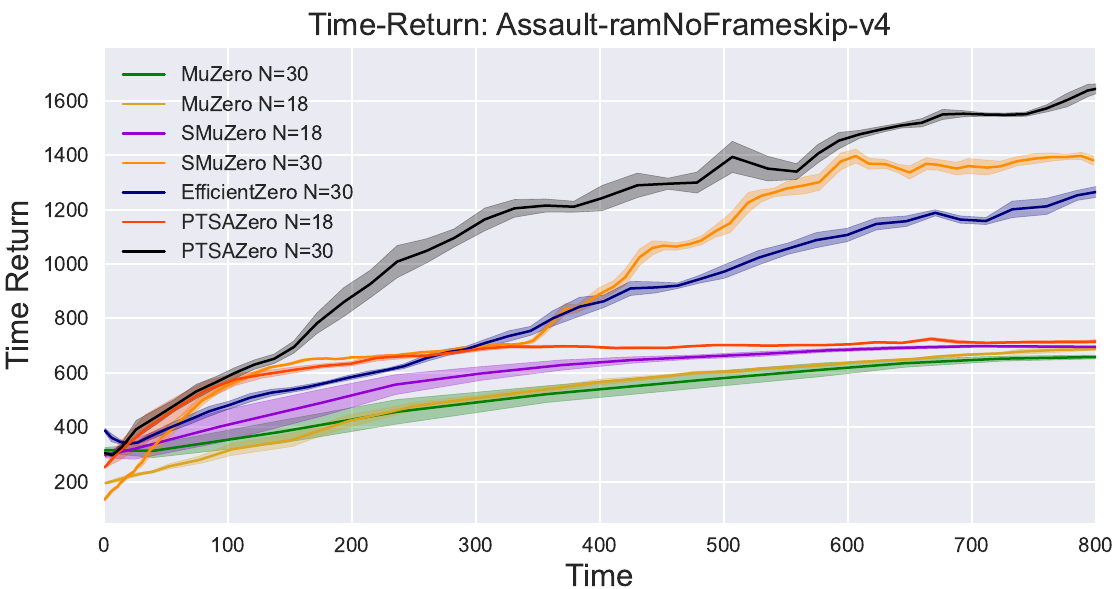}}
\subfigure[Seaquest] {\includegraphics[width=0.46\linewidth]{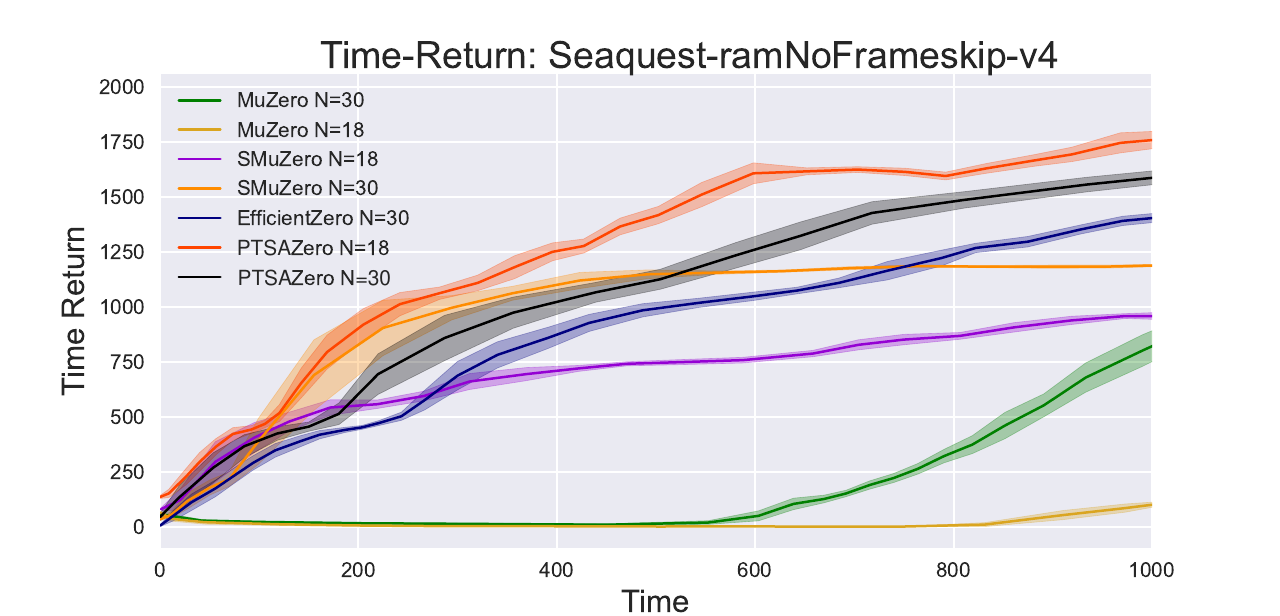}}
\subfigure[MsPacman] {\includegraphics[width=0.45\linewidth]{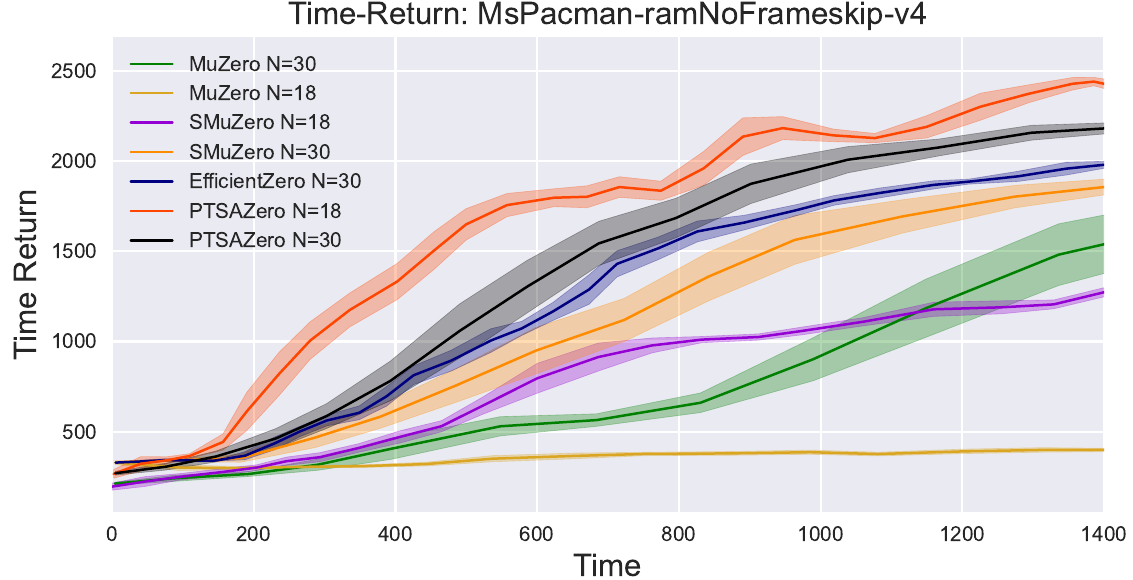}}
\subfigure[Breakout] {\includegraphics[width=0.45\linewidth]{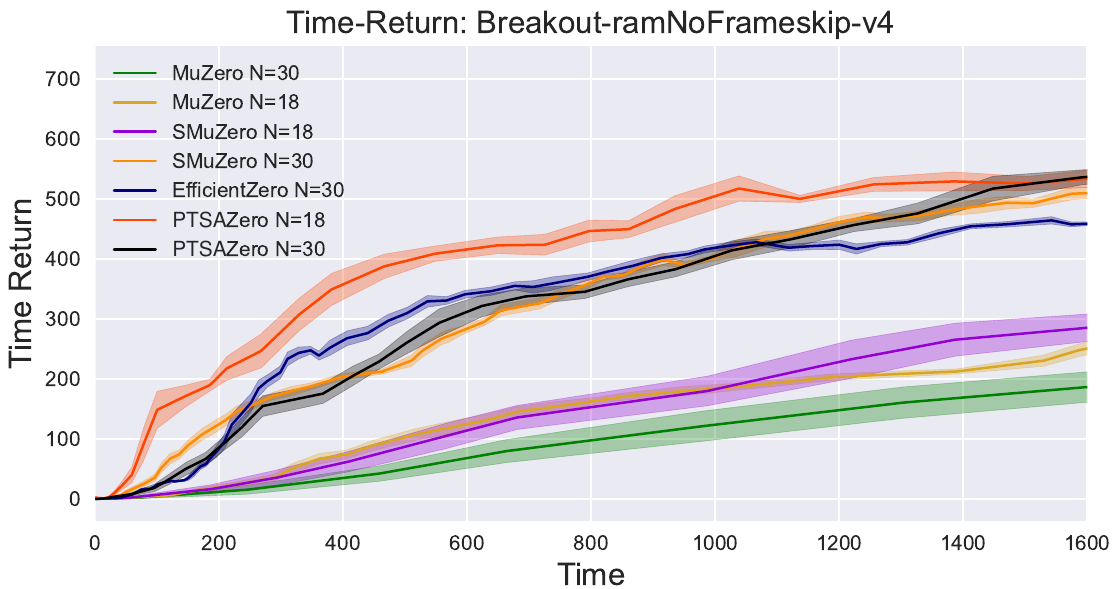}}

\caption{More experimental results on Atari benchmarks.} 
\label{fig:Atari2}
\end{figure*}

\section{Hyperparameters}
Conducting experiments on Atari game tasks, the setting of hyperparameters is shown in Table.\ref{tab:Atari Parameter}. Typically, hyperparameters include learning rate, optimizer, batch size, discount factor, experience replay buffer size, and more. The frame size of the Atari game denotes the pixel size of the observation. In the MuZero-based algorithm, each actor can interact with the environment and collect experience independently, which can increase the amount of experience and reduce the time needed for learning. To ensure equal parallel processing capabilities across all algorithms, we have set the number of actors to 7 for each method. This uniform setting helps to ensure that each method can effectively utilize parallel processing resources.

\section{Additional Experiments}
For a clear numerical comparison, Table \ref{tab:update frames} shows the average computation time of collecting 1k frames with different simulation times on Atari benchmarks. Compared to other algorithms, PTSA introduces an acceptable decrease in trajectory collection efficiency (less than $8\%$ on average), which results in a significant reduction in the whole training time. Additionally, we compare different $\alpha$ in probability tree state abstraction, and results are shown in Figure \ref{fig:ablation}. Results demonstrate that the algorithm's convergence speed improves as the parameter $\alpha$ increases.

Moreover, the comparison results of average search depth between SMuZero and PTSAZero with different $\alpha$ are shown in Figure \ref{fig:search depth}. Since the search space of MCTS is reduced by tree state abstraction, the search depth of PTSAZero is deeper than that of SMuZero with same number of simulations. 

\begin{table}[t!]
\caption{Average computation time (seconds) of collecting 1k frames in Atari benchmarks. Box. denotes Boxing, Free. denotes Freeway, Ten. denotes Tennis, Break. denotes Breakout, MsP. denotes MsPacman, and Sea. denotes Seaquest tasks respectively. Ave. denotes Average computation time. }
  \centering
    \begin{tabular}{c|cccccccccc}
    \hline
    Methods  & \multicolumn{1}{c}{Box.} & \multicolumn{1}{c}{Free.} & \multicolumn{1}{c}{Pong} & \multicolumn{1}{c}{Alien} & \multicolumn{1}{c}{Ten.} & \multicolumn{1}{c}{Assault} & \multicolumn{1}{c}{Break.} & \multicolumn{1}{c}{MsP.} & \multicolumn{1}{c}{Sea.} & \multicolumn{1}{c}{Ave.}\\
    \hline
    \textit{MuZero N=30}  &6.31	&3.47	&4.56	&8.85 &3.86 &3.41 &3.24	&3.43 &3.18 & 4.48\\
    \textit{SMuZero N=30}  &6.89 &4.43  &4.44	&8.89 &4.02 &3.47 &3.33 &3.50 &3.35	 &4.70  \\
    \textit{PTSAZero N=30} &6.74 &3.85	&4.81	&9.04 &4.11 &3.58 &3.50	&3.63 &3.37  &4.74   \\
    
    \hline
    \textit{MuZero N=18}  &3.06	 &2.39	&2.46	&3.49	&3.42  &1.71   &1.61 &1.92 &1.69 &2.42	\\
    \textit{SMuZero N=18} &4.12	 &2.41	&2.11	&3.07	&3.45  &1.76   &1.97 &1.98 &1.72	&2.51	\\
    \textit{PTSAZero N=18} &3.41 &1.96	&3.16	&3.37	&3.55  &1.85   &1.86 &2.05 &1.84	&2.56\\
    \hline
    \end{tabular}
  \label{tab:update frames}%
\end{table}

\end{document}